\newcommand{\arxivversion}{}

\ifdefined\arxivversion
  \PassOptionsToClass{nopreprintline}{elsarticle}
\fi
\documentclass{elsarticle}

\usepackage{iftex}
\ifPDFTeX
  \usepackage[T1]{fontenc}
  \usepackage[utf8]{inputenc}
  \usepackage{newtxtext}
\else
  \usepackage{fontspec}
\fi
\ifdefined\arxivversion
  \usepackage[a4paper,margin=2.5cm]{geometry}
\else
  \usepackage[top=4.3cm,right=4.8cm,bottom=4.3cm,left=4.8cm]{geometry}
 \usepackage{lineno}
 \linenumbers
  \input{submission-metadata}
 \usepackage{setspace}
\fi
\usepackage{amssymb,mathtools}
\usepackage{amsthm}
\usepackage{microtype}
\usepackage{array,booktabs,multirow}
\usepackage{float}
\usepackage{enumitem}
\setlist[enumerate]{leftmargin=.5in}
\setlist[itemize]{leftmargin=.5in}
\usepackage{tikz}
\usetikzlibrary{arrows.meta,positioning,fit}

\ifdefined\arxivversion
  \usepackage{xurl}
\fi
\usepackage[colorlinks=false,pdfborder={0 0 0},breaklinks=true]{hyperref}
\usepackage{xr-hyper}
\usepackage[capitalize,nameinlink]{cleveref}

\newcommand{\TopNMF}{Top-NMF}
\ifdefined\arxivversion
  \newcommand{\suppref}[1]{\cref{#1}}

\else
  \newcommand{\suppref}[1]{supplement, \cref{#1}}

\fi

\newcommand{\R}{\mathbb{R}}
\newcommand{\Z}{\mathbb{Z}}
\newcommand{\Rnn}{\mathbb{R}_{\ge 0}}

\newcommand{\X}{X}
\newcommand{\W}{W}
\newcommand{\V}{V}
\newcommand{\Lapx}{L_{\mathrm{recon}}}
\newcommand{\Ltop}{L_{\mathrm{top}}}
\newcommand{\lambdatop}{\lambda_{\mathrm{top}}}

\newcommand{\Fspace}{\mathcal{F}}
\newcommand{\Fil}{\mathfrak{F}}
\newcommand{\KOm}{K^{\Omega}}

\DeclareMathOperator{\PH}{PH}
\DeclareMathOperator{\supp}{supp}

\DeclareMathOperator{\PerS}{PS}
\DeclareMathOperator{\TP}{TP}
\DeclareMathOperator{\WTP}{WTP}
\DeclareMathOperator{\CP}{CP}
\DeclareMathOperator{\SW}{SW}

\theoremstyle{plain}
\newtheorem{theorem}{Theorem}[section]

\newtheorem{corollary}[theorem]{Corollary}
\newtheorem{proposition}[theorem]{Proposition}

\theoremstyle{remark}

\newtheorem{example}[theorem]{Example}

\ifdefined\arxivversion
\else
\journal{\submissionjournal}
\fi

\begin{document}

\begin{frontmatter}

\title{\fontsize{14}{17}\selectfont Topological Regularisation of Non-negative Matrix Factorisation}

\author[1]{Matias de Jong van Lier}
\ead{matias.vanlier@recursiveai.co.jp}

\author[2]{Shizuo Kaji\corref{cor1}}
\ead{kaji.shizuo.7r@kyoto-u.ac.jp}

\author[3]{Keunsu Kim}
\ead{kim.keunsu.752@m.kyushu-u.ac.jp}

\cortext[cor1]{Corresponding author.}

\address[1]{Recursive Inc., Tokyo 150-0002, Japan}
\address[2]{Graduate School of Science, Kyoto University, Kyoto 606-8502, Japan}
\address[3]{Institute of Mathematics for Industry, Kyushu University, Fukuoka 819-0395, Japan}

\begin{abstract}
Non-negative matrix factorisation (NMF) learns additive representations
from data, but non-negativity alone does not ensure interpretable basis
functions. We introduce Top-NMF, which guides basis learning through
general topological preferences without prescribing the detailed form of
the components. Observations and basis vectors are treated as non-negative
functions on structured domains. Discrete topological conditions are
difficult to optimise directly. Our guiding viewpoint is that persistent
homology provides a stable, continuous relaxation of ordinary homological
invariants: it tracks connected components and loops across thresholds,
allowing structural preferences to be expressed through continuous scores
suitable for optimisation. We construct such scores for connected image parts,
clique-like graph structure, and periodic time-series components, and
incorporate them into a common NMF objective. For graph data, we derive an
exact maximum-weight-spanning-tree formula and characterise all local
minimisers of the score, showing that their positive-weight edges form
disjoint cliques. We establish the regularity of the scores, describe their
derivatives, and analyse projected first-order optimisation. Numerical
studies on synthetic and real data demonstrate how these priors guide basis
learning and examine the balance between structural agreement, atom
recovery, and reconstruction accuracy.
\end{abstract}


\end{frontmatter}

\section{Introduction}
\label{sec:introduction}

Linear representations provide a common framework for dimensionality
reduction and feature extraction. Classical Fourier and wavelet
representations use prescribed basis functions, with coefficients that
describe frequency content or localised variations at different scales
\cite{mallat1999wavelet}. In contrast, principal component analysis (PCA)
\cite{jolliffe2002principal}, non-negative matrix factorisation (NMF)
\cite{lee1999learning}, and dictionary learning \cite{mairal2010online}
learn their components from data. This flexibility allows the
representation to adapt to the observed patterns, but the learned
components need not have an immediately recognisable meaning. Improving
their interpretability is therefore an important consideration alongside
reconstruction accuracy.

Regularisation can guide this learning process by expressing general
structural preferences without prescribing individual basis functions.
Sparsity, for example, favours a small number of active entries while
leaving their locations to be learned \cite{hoyer2004non}.
We consider topological simplicity as another such preference. An image
component may be expected to form a connected region, even when its shape
and location are unknown. Such a requirement concerns how active entries
are arranged, which sparsity alone does not capture. It can therefore
guide the learning of interpretable components while preserving
flexibility in their detailed form.

Turning this preference into an optimisation objective is less
straightforward. The numbers of connected components and holes in a
thresholded image are integer-valued and can change abruptly under small
perturbations. Persistent homology instead tracks these features across
thresholds and records how long they persist
\cite{edelsbrunner2010computational}. Their lifetimes provide graded
measures of topological structure from which stable, continuous penalties
can be constructed \cite{cohen2007stability}.
We view this passage from discrete feature counts to continuous structural
scores as a \emph{continuous relaxation of ordinary homological
invariants}. It allows topological preferences to be incorporated into
basis learning through regularisation, with suitable scores admitting
gradient-based optimisation \cite{carriere2021optimizing}.

We develop this viewpoint within NMF, which represents observations as
non-negative combinations of learned basis vectors
\cite{lee1999learning,lee2000algorithms}.
Observations and basis vectors are treated as non-negative functions
on a structured domain, such as a pixel grid, a set of graph edges, or a
time axis. The domain determines how to construct a filtration, a nested
family of spaces on which persistent homology is computed. A scalar score
of this persistent homology then becomes a regularisation term alongside
the reconstruction loss. The resulting framework, \TopNMF{}, makes it
possible to choose a structural prior while retaining the familiar NMF
representation.

Connectivity is a basic instance of this construction, related to
persistence-based losses for image segmentation
\cite{hu2019topology,clough2022topological}. 
We use it to promote coherent
image parts and extend the construction to clique-like graph bases and
periodic time-series components. 
\Cref{sec:related-work} situates this approach within structured NMF and
persistence-based learning.

We establish
the regularity needed to optimise all three scores, describe their
derivatives, and analyse the computational cost. Experiments on images,
graphs, and time series examine how the priors influence basis structure,
recovery of generating components, and reconstruction accuracy. Contact
networks and multi-city temperature data provide applications with
structural diagnostics based on class labels and Fourier spectra.

\paragraph{Organisation}
\Cref{sec:related-work} reviews related methods.
\Cref{sec:functional-viewpoint} gives our formulation, and
\cref{sec:support-to-ph} explains the passage from thresholded topology to
continuous persistence scores. \Cref{sec:data-type-topological-scores}
defines the three scores and states the graph characterisation.
\Cref{sec:optimization} treats optimisation and computational cost, and
\cref{sec:numerical-validation} reports the numerical results.
\ifdefined\arxivversion The appendices provide\else
The supplementary material provides\fi{} proofs, experimental details, and
sensitivity analyses. Code is available at
\url{https://github.com/shizuo-kaji/TopNMF}.

\section{Related Work}
\label{sec:related-work}

\paragraph{Structured NMF}
NMF represents additive non-negative data by non-negative coefficients and
bases \cite{lee1999learning}; see \cite{gillis2020nonnegative} for a systematic
account of its models and algorithms.  Regularised variants encode additional
interpretability assumptions.
Coefficient-side methods, including graph-regularised NMF, preserve
neighbourhoods or manifold geometry in the representation space
\cite{cai2010graph}.  Basis-side methods
instead promote properties such as sparsity, spatial
localisation, or spatial smoothness
\cite{hoyer2004non,li2001learning,yin2010nonnegative}, or spectral properties
of the learned factors \cite{kim2024supervised}.  Other structured approaches
impose unimodality and non-negativity directly in alternating least squares
\cite{bro1998least}.  When the data
themselves are graphs, symmetric NMF factorises an adjacency or similarity
matrix to obtain a clustering representation
\cite{kuang2012symmetric}.
\TopNMF{} complements them by regularising the basis rows through filtrations
defined on the feature domain.

\paragraph{Persistent homology and matrix factorisation}
Topological summaries have also been combined with NMF and related data
analysis methods, for example through persistent homology descriptors of voxel and other
scientific data \cite{obayashi2022persistent}.  In those pipelines, persistent homology is
computed before the factorisation.
Persistent topology has also been incorporated directly into NMF objectives
such as in TNMF, where persistent-Laplacian regularisation
of the sample representation for single-cell RNA sequencing is considered
\cite{hozumi2024analyzing}.  

\paragraph{Persistent homology in optimisation}
Persistent homology has also been used inside learning objectives rather
than only as a post hoc descriptor; see \cite{hensel2021survey} for a survey of
topological machine learning.  An early instance drives a point cloud towards prescribed
persistent homology by continuation \cite{gameiro2016continuation}.
Persistence-based losses have since been
studied for geometric reconstruction, classification, and representation
learning
\cite{bruel2020topology,chen2019topological,moor2020topological},
and differentiable persistence layers are available as general-purpose
components \cite{gabrielsson2020topology}.  A complementary line of work learns
the filtration itself rather than the object being filtered
\cite{hofer2020graph}.
Our grid score is related to the topological losses for image segmentation used
in \cite{hu2019topology,clough2022topological}, which compare the
persistent homology of a predicted segmentation with prescribed topological
targets.
Our analysis uses variational results for persistence objectives
\cite{carriere2021optimizing,leygonie2022framework}.

\section{Functional Viewpoint and Problem Setup}
\label{sec:functional-viewpoint}

We treat observations and basis elements as non-negative functions on a
structured domain. For an image, the domain consists of pixel locations
together with their grid adjacency; for graph data, the features are edges
with a specified incidence structure; for a time series, the samples have
a temporal order. 
Accounting for this underlying structure allows us to capture the topology of patterns, which is inevitably lost when data are flattened into vectors.

Let $\Omega=\{p_1,\ldots,p_d\}$ be the feature locations and write
$\Fspace(\Omega):=\{f:\Omega\to\Rnn\}$.
Given observations $x_1,\ldots,x_n\in\Fspace(\Omega)$, we seek basis
functions $v_1,\ldots,v_r\in\Fspace(\Omega)$ and coefficients $w_{ij}\geq0$
such that
\begin{equation}
x_i \approx \sum_{j=1}^r w_{ij}v_j,
\qquad i=1,\ldots,n.
\label{eq:functional-factorization}
\end{equation}
Identifying a function with its evaluation vector
$\bigl(f(p_1),\ldots,f(p_d)\bigr)$ and stacking rowwise gives the data
matrix $\X\in\Rnn^{n\times d}$, basis matrix $\V\in\Rnn^{r\times d}$, and
coefficient matrix $\W\in\Rnn^{n\times r}$. Thus
\eqref{eq:functional-factorization} is the usual factorisation
$\X\approx\W\V$. We use the usual squared reconstruction loss in this paper,
though our framework is independent of the choice:
\begin{equation}
\Lapx(\X,\W\V):=\|\X-\W\V\|_F^2.
\label{eq:approximation-loss}
\end{equation}

To encode the domain structure, we associate a finite 
\emph{complex} $\KOm$ to $\Omega$. A complex consists of vertices, edges, and,
where appropriate, higher-dimensional cells: for instance, neighbouring
pixels form edges and groups of four pixels form squares on an image grid.
A subcomplex retains cells together with all their faces. We use these
spaces to ask whether a region is connected or contains a hole; background
terminology is given in
\cite{edelsbrunner2010computational,dey2022computational}.

A basis function $v$ determines a \emph{filtration} $\Fil(v)$, a nested
family of subcomplexes indexed by a threshold. For a vertex-valued signal
or image, retain the vertices with value at least $\tau$, together with
every cell whose vertices have all been retained. This is the upper-star
superlevel filtration: setting
\[
\hat v(\sigma):=\min\{v(p):p\text{ is a vertex of }\sigma\},
\]
we obtain
\begin{equation}
\KOm_\tau(v):=\{\sigma\in\KOm:\hat v(\sigma)\geq\tau\},
\qquad \tau\geq0.
\label{eq:upper-star-superlevel}
\end{equation}
As $\tau$ decreases, more cells enter and the retained region grows.
For edge-valued graphs and sliding-window point clouds, the vertices of
$\KOm$ index graph vertices and time windows, respectively, while
$\Omega$ indexes edges and time samples. Their filtrations are defined in
\cref{subsec:graph-data,subsec:time-series-latent-topology}.

To quantify the desired structure, choose a scalar function $q$ of the
\emph{persistent homology} (\Cref{sec:support-to-ph}) of this filtration and set
\[
\mathcal Q(v):=q\bigl(\PH(\Fil(v))\bigr).
\]
The filtration determines which structure is measured, and $q$ determines
which features are rewarded or penalised. Given the index set
$\mathcal J_{\mathrm{top}}\subseteq\{1,\ldots,r\}$ of regularised basis
rows, define
\begin{equation}
\Ltop(\V):=\sum_{j\in\mathcal J_{\mathrm{top}}}\mathcal Q_j(v_j).
\label{eq:topological-loss}
\end{equation}
The regularised loss has the common form
\begin{equation}
\Lapx(\X,\W\V)+\lambdatop\Ltop(\V),
\label{eq:topnmf-objective}
\end{equation}
where $\lambdatop\geq0$ balances reconstruction and the structural prior.
We next explain how persistent homology turns discrete structural
conditions into continuous scores, before giving the three constructions.

\section{Persistent Homology as a Continuous Relaxation}
\label{sec:support-to-ph}

Consider a basis image with two bright regions separated by a dimmer
bridge. At a high threshold the regions are disconnected; lowering the
threshold eventually includes the bridge and joins them. Persistent
homology records the thresholds at which components appear and merge.
The lifetime of an extra component measures its prominence, so penalising
long lifetimes encourages connected parts across thresholds. This gives
a concrete interpretation of persistent homology as a continuous
relaxation of discrete topological information.

We formalise this idea for the vertex-valued superlevel filtration in
\eqref{eq:upper-star-superlevel}. For $v\in\Fspace(\Omega)$, its support is
$\supp(v):=\{\omega\in\Omega:v(\omega)>0\}$ and its thresholded support is
\begin{equation}
\supp_\tau(v):=\{\omega\in\Omega:v(\omega)\geq\tau\},
\qquad \tau\geq0.
\label{eq:thresholded-support}
\end{equation}
The domain complex supplies the adjacency and cells, giving the
thresholded space $\KOm_\tau(v)$. Ordinary homology describes its
independent topological features: dimension $0$ counts connected
components, and dimension $1$ counts loops surrounding holes. We write
$H_k(Y)$ for homology over a fixed field and
$\beta_k(Y):=\dim H_k(Y)$ for the corresponding Betti number.
Reduced homology counts the \emph{extra} components beyond the first:
for a non-empty space $Y$, $\dim\widetilde H_0(Y)=\beta_0(Y)-1$.
In particular, a non-empty space is connected exactly when
$\widetilde H_0(Y)=0$.

At a fixed threshold, these integer-valued counts can change abruptly
when a value of $v$ crosses the threshold. Persistent homology instead
follows the features through the nested family
\begin{equation}
\Fil_{\mathrm{sup}}(v):=\{\KOm_\tau(v)\}_{\tau\geq0}.
\label{eq:superlevel-filtration}
\end{equation}
The inclusions $\KOm_{\tau_2}(v)\hookrightarrow\KOm_{\tau_1}(v)$ for
$\tau_1\leq\tau_2$ induce linear maps on homology, which track features
as cells enter. We represent the resulting persistent homology
$\PH_k(v)$ by a multiset of \emph{birth--death pairs} $(b,d)$.
A feature appears at threshold $b$ and ceases to contribute an independent
class at threshold $d$; for components, a death occurs when two components
merge. In this descending filtration a finite pair has $b>d$ and
contributes over the interval $(d,b]$. Zero-length intervals are omitted.
Features that remain in the terminal complex are called \emph{essential};
for example, one component remains when the full domain is connected.
We extend the filtration constantly below zero and assign these features
$d=-\infty$. In the ascending Vietoris--Rips filtration used for time
series, finite pairs instead have $b<d$.

The finite lifetime $|b-d|$ measures how long a feature persists.
Stability of persistent homology \cite{cohen2007stability} ensures that
small changes in filtration values cause small changes in the persistent
homology. Thus lifetimes replace discrete feature counts with continuously
varying quantities from which optimisation objectives can be constructed.
For further background, see
\cite{edelsbrunner2010computational,dey2022computational}.
A simple score is the \emph{total squared persistence}
\begin{equation}
\TP^{(k)}(v):=
\sum_{(b,d)\in\PH_k^{\mathrm{fin}}(v)}(b-d)^2,
\label{eq:total-squared-persistence}
\end{equation}
where $\PH_k^{\mathrm{fin}}(v)$ denotes the finite birth--death pairs of
\eqref{eq:superlevel-filtration}. It penalises long-lived extra components
when $k=0$ and long-lived loops when $k=1$.

For a small example, take the path on $\Omega=\{1,\ldots,5\}$, with
edges between consecutive vertices, and compare the basis pairs
\[
\begin{aligned}
v_1 &= (1,1,0,0,0), & v_2 &= (0,0,0,1,1);\\
v_1' &= (0,1,0,1,0), & v_2' &= (1,0,0,0,1).
\end{aligned}
\]
Both pairs sum to $(1,1,0,1,1)$ and give the same reconstruction with
coefficients $(1,1)$. Their rows differ only by coordinate permutations,
so a permutation-invariant measure, such as sparsity, gives the same score to both
pairs. Their connectivity is different. Each unprimed row has one
component born at $1$ that remains essential, giving $\TP^{(0)}=0$.
Each primed row has two components born at $1$ that merge at $0$; one
remains essential, while the finite pair $(1,0)$ contributes $1$ to
$\TP^{(0)}$. The topological loss therefore prefers the first pair.

The next result makes the viewpoint of a continuous relaxation precise
for this score: its zero set encodes a topological condition at every
threshold. 

\begin{proposition}
\label[proposition]{prop:ph-support-topology}
Fix $k\geq0$ and suppose $\widetilde H_k(\KOm)=0$.
Then, for every $v\in\Fspace(\Omega)$,
\[
\TP^{(k)}(v) = 0
\quad\Longleftrightarrow\quad
\widetilde{H}_k(\KOm_\tau(v)) = 0
\text{ for every } \tau \in [0,\|v\|_\infty],
\]
where $\|v\|_\infty:=\max_{\omega\in\Omega}|v(\omega)|$.
\end{proposition}

\begin{proof}
Under the hypothesis, no essential class contributes to reduced homology
in dimension $k$, so the birth--death pairs give
\begin{equation}
\operatorname{rank}\widetilde H_k(\KOm_\tau(v))
=\#\{(b,d)\in\PH_k^{\mathrm{fin}}(v):d<\tau\leq b\}.
\label{eq:bar-counting}
\end{equation}
Every retained pair has $b>d$, so these ranks vanish at all thresholds
exactly when $\PH_k^{\mathrm{fin}}(v)=\varnothing$, or equivalently
$\TP^{(k)}(v)=0$.
\end{proof}

\section{Topological Scores by Domain}
\label{sec:data-type-topological-scores}

We now instantiate the construction with persistence-based scores for
desired structural properties. Starting with connectivity and the absence
of holes on grids, we develop scores for clique-like structure on graphs
and periodicity in time series. Each score serves as a regulariser for
basis learning through \eqref{eq:topological-loss}.
\Cref{tab:domain-scores} summarises what each filtration measures and how
the score acts on it. Their regularity is used in \cref{sec:optimization}.

\begin{table}[htbp]
\centering
\small
\renewcommand{\arraystretch}{1.15}
\begin{tabular}{@{}>{\raggedright\arraybackslash}p{0.18\linewidth}
                    >{\raggedright\arraybackslash}p{0.34\linewidth}
                    >{\raggedright\arraybackslash}p{0.40\linewidth}@{}}
\toprule
Domain & Structure measured & Effect of regularisation\\
\midrule
Images and signals & Connectivity and holes of superlevel sets &
Penalise extra components and holes with $\TP$ or $\WTP$.\\
Graphs & Components and cycles as edges enter &
Reward separation between groups and cycles within them with $\CP_\alpha$.\\
Time series & Loops in sliding-window point clouds &
Match basis periodicity to a specified target.\\
\bottomrule
\end{tabular}
\caption{Three ways to turn topological structure into a basis regulariser.}
\label{tab:domain-scores}
\end{table}

\subsection{Grid-valued scalar fields}
\label{subsec:support-based-data}

For a signal or image on a line or pixel grid, use the upper-star superlevel
filtration $\Fil_{\mathrm{sup}}(v)$ of \eqref{eq:superlevel-filtration} and the
total squared persistence $\TP^{(k)}$ of
\eqref{eq:total-squared-persistence}.  In one dimension, $\KOm$ is the path
complex on $\{1,\ldots,d\}$; for an image, pixels are the vertices of the
cubical grid, with edges and squares as higher cells.  The cases $k=0$ and $k=1$ penalise persistent
extra components and loops, respectively. The font glyph example in
\cref{subsec:hangul-main} illustrates the connectivity score; the synthetic pattern experiment in \suppref{subsubsec:exp-bars} uses both dimensions.

The score is Lipschitz on
bounded sets.  On a region with a fixed weak ordering of the vertex values,
choose a compatible cell ordering.
Boundary-matrix reduction fixes the birth--death cell pairing throughout
that region.  The same ordering remains valid on its closure; pairs whose
endpoints coincide there have zero contribution.  Thus the score is
continuous and piecewise quadratic, including where the pairing changes.
More explicitly, for
$v,v'\in[0,R]^d$ it holds
\begin{equation}
\left|\TP^{(k)}(v)-\TP^{(k)}(v')\right|
\leq4|\KOm|R\,\|v-v'\|_\infty.
\label{eq:tp-lipschitz-bound}
\end{equation}

We also consider
a weighted variant of $\TP$.  For a maximum-normalised input
$u\in[0,1]^d$ with $\|u\|_\infty=1$, define
\begin{equation}
\WTP_p^{(k)}(u)
:=\sum_{(b,d)\in\PH_k^{\mathrm{fin}}(u)}
(1-d)^p(b-d)^2,
\qquad p\in\Z_{\geq0}.
\label{eq:weighted-total-squared-persistence}
\end{equation}
Lower death values receive greater weight. Since every finite pair satisfies $0\leq d<b\leq1$, the factor
$(1-d)^p$ is positive, so $\WTP_p^{(k)}$ vanishes exactly when there are no finite birth--death pairs.
It therefore has the same zero set as $\TP^{(k)}$, characterised in
\cref{prop:ph-support-topology}.
The weighting suppresses
well-separated secondary peaks more strongly than shallow ripples.
The score $\WTP_2^{(0)}$ is used to promote unimodality
(\suppref{subsubsec:exp-1d-unimodal}).
The same pairing argument makes $\WTP_p^{(k)}$ continuous and piecewise
polynomial, hence locally Lipschitz.

\subsection{Clique-promoting score for graph data}
\label{subsec:graph-data}

Let $G=(U,E)$ be the complete graph on $|U|\geq3$ vertices and let
$v:E\to[0,1]$ be an edge weighting.  We call $v$ \emph{admissible} when
$\|v\|_\infty=1$.  
Define a descending filtration
$\Fil_G(v):=\{G_\tau(v)\}_{0\leq\tau\leq1}$ with
\begin{equation}
G_\tau(v):=\bigl(U,\{e\in E:v(e)\geq\tau\}\bigr),
\qquad 0\leq\tau\leq1,
\label{eq:graph-edge-filtration}
\end{equation}
where all vertices enter at value one and an edge $e$ at value
$v(e)$. We retain only vertices and edges, so triangles and other cycles
contribute to the graph's cycle space. These cycles capture redundant
connections within a group of vertices.  This is the graph instance of the filtration of
\cref{sec:functional-viewpoint}: here $\KOm$ is $G$ regarded as a
one-dimensional complex and $\Omega$ indexes its edges. The filtration is
defined directly by \eqref{eq:graph-edge-filtration}, with vertices present
at every $\tau\le1$, rather than by the vertex-valued upper-star rule.

Let
$\PH_0^{\mathrm{fin}}(v)$ denote the finite birth--death pairs in dimension $0$, and let
$\PH_1^{\mathrm{tr}}(v)$ denote the birth--death pairs in dimension $1$ with their death
coordinates truncated to zero.  For
$\alpha>0$, define
\begin{equation}
\CP_\alpha(v):=
-\sum_{(b,d)\in\PH_0^{\mathrm{fin}}(v)}(b-d)^2
-\alpha\sum_{(b,d)\in\PH_1^{\mathrm{tr}}(v)}(b-d)^2.
\label{eq:clique-promoting-functional}
\end{equation}
Minimisation rewards both persistent separation between vertex groups and
strong cycles within them.  The parameter $\alpha$ balances these effects.
As shown below, the strict local minima are exactly the cluster graphs
characterised in \cref{thm:cp-local-minimizers}, which motivates the use of
\eqref{eq:clique-promoting-functional} as a
basis prior in the synthetic example of \cref{subsec:graph-main}
and the contact-network study in \suppref{subsec:exp-graphs-sociopatterns}.

Proofs of the theorems are given in
\suppref{sec:proofs}.

\subsubsection{Maximum spanning-tree representation}

Recall Kruskal's algorithm, which processes edges in non-increasing order and accepts an
edge exactly when it does not create a cycle.  Its possible outputs, over
all tie-breaking rules, are exactly the maximum-weight spanning trees
(MWSTs) \cite{korte2018combinatorial}.

\begin{proposition}
\label[proposition]{prop:kruskal-representation}
Let $T(v)$ be any MWST of $(G,v)$.  Then
\begin{align}
\sum_{(b,d)\in\PH_0^{\mathrm{fin}}(v)}(b-d)^2
&=\sum_{e\in T(v)}(1-v(e))^2,
\label{eq:ph0-kruskal}\\
\sum_{(b,d)\in\PH_1^{\mathrm{tr}}(v)}(b-d)^2
&=\sum_{e\in E\setminus T(v)}v(e)^2.
\label{eq:ph1-kruskal}
\end{align}
Consequently,
\begin{equation}
\CP_\alpha(v)
=-\sum_{e\in T(v)}(1-v(e))^2
-\alpha\sum_{e\in E\setminus T(v)}v(e)^2.
\label{eq:cp-kruskal}
\end{equation}
Equivalently, writing $\mathcal T$ for the set of all spanning trees of $G$,
\begin{equation}
\CP_\alpha(v)=\max_{T\in\mathcal T}q_T(v),
\qquad
q_T(v):=-\sum_{e\in T}(1-v(e))^2
        -\alpha\sum_{e\notin T}v(e)^2.
\label{eq:cp-tree-envelope}
\end{equation}
\end{proposition}

Equation \eqref{eq:ph0-kruskal} expresses the standard correspondence between
$0$-dimensional persistent homology of an edge filtration and a maximum spanning tree
\cite[\S4]{bauer2021ripser}; \eqref{eq:ph1-kruskal} is its cycle-space
counterpart, the non-tree edges indexing a basis of $H_1$.
One run of Kruskal's algorithm therefore evaluates $\CP_\alpha$ in
$\mathcal O(|E|\log|E|)$ time.  As a finite maximum of quadratic functions,
$\CP_\alpha$ is locally Lipschitz.

\begin{corollary}
\label[corollary]{cor:cp-binary}
Let $E_1\subseteq E$ be non-empty and let $c$ be the number of connected
components of $(U,E_1)$, isolated vertices included.  Then
\begin{equation}
\CP_\alpha(\mathbf 1_{E_1})
=-(c-1)-\alpha\bigl(|E_1|-|U|+c\bigr)
=-\bigl(\beta_0(U,E_1)-1\bigr)-\alpha\,\beta_1(U,E_1).
\label{eq:cp-binary}
\end{equation}
\end{corollary}

\begin{example}
\label[example]{ex:alpha-effect}
On the complete graph $K_4$, let $v_\triangle$ be a unit-weight triangle with one
isolated vertex, and let $v_{K_4}$ be the unit-weight complete graph.  Their
component and cycle counts are $(c,\beta_1)=(2,1)$ and $(1,3)$, so
\eqref{eq:cp-binary} gives
\[
\CP_\alpha(v_\triangle)=-1-\alpha,
\qquad
\CP_\alpha(v_{K_4})=-3\alpha.
\]
Therefore, the full clique $v_{K_4}$ is preferred exactly when $\alpha>1/2$.  Both configurations
are strict local minimisers by \cref{thm:cp-local-minimizers}.
\end{example}

\subsubsection{Cluster-graph characterisation}

A graph is a \emph{cluster graph} if it is a disjoint union of cliques; such
graphs are the targets of the classical cluster-editing problems
\cite{shamir2004cluster}.  Call a
connected component \emph{non-trivial} if it contains an edge, and write
$\mathbf 1_{E_1}$ for the indicator of $E_1\subseteq E$.  We consider local
minimality, in the $\ell_\infty$ metric, relative to
\[
\mathcal A:=\{v\in[0,1]^E:\|v\|_\infty=1\}.
\]

\begin{samepage}
\begin{theorem}
\label{thm:cp-local-minimizers}
Let $\alpha>0$ and $v\in\mathcal A$.  Then $v$ is a strict local minimiser of
$\CP_\alpha$ on $\mathcal A$ if and only if $v=\mathbf 1_{E_1}$ and
$(U,E_1)$ is a cluster graph satisfying one of the following conditions:
\begin{enumerate}
\item every non-trivial component has at least three vertices; or
\item $E_1$ consists of a single edge.
\end{enumerate}

Every local minimiser is strict.  Every such configuration also satisfies
\begin{equation}
\CP_\alpha(v+w)-\CP_\alpha(v)
\geq 2\min\{1,\alpha\}\|w\|_\infty
-|E|\max\{1,\alpha\}\|w\|_\infty^2
\label{eq:cp-local-growth}
\end{equation}
whenever $0<\|w\|_\infty<1/2$ and $v+w\in\mathcal A$.
\end{theorem}
\end{samepage}

\begin{corollary}
\label[corollary]{cor:cp-global-minimizers}
Let $N=|U|\geq3$ and $\alpha>0$.  Then
\begin{equation}
\min_{v\in\mathcal A}\CP_\alpha(v)
=-\max\left\{N-2,\frac{\alpha(N-1)(N-2)}2\right\}.
\label{eq:cp-global-minimum}
\end{equation}
If $\alpha<2/(N-1)$, the global minimisers are exactly the unit-weight
single-edge graphs; if $\alpha>2/(N-1)$, the unique global minimiser is the
unit-weight complete graph.  At equality, both types, and no others,
are global minimisers.
\end{corollary}

\subsection{Periodicity of time series}
\label{subsec:time-series-latent-topology}
For time series, we examine the geometry traced by short sequences of
successive samples. Each such window is a point in a higher-dimensional
space. As a periodic pattern repeats, similar windows recur, and the
trajectory can close into a loop. With suitable embedding parameters and
sufficient sampling, this loop is detected by persistent homology.

Let $v:\{1,\ldots,d\}\to\R$ be a signal, and choose
$M,h\in\mathbb N$ with $Mh<d$.  Its delay-$h$ sliding-window embedding into
$\R^{M+1}$ is
\begin{equation}
\SW_{M,h}v(t)
:=\bigl(v(t),v(t+h),\ldots,v(t+Mh)\bigr),
\qquad 1\leq t\leq d-Mh.
\label{eq:sliding-window-embedding}
\end{equation}
Write $z_t=\SW_{M,h}v(t)$,
$\mu_t=(M+1)^{-1}\mathbf 1^\top z_t$, and let
$\mathbf 1\in\R^{M+1}$ be the all-ones vector.  For a fixed small $\rho>0$,
set
\begin{equation}
y_t:=
\frac{z_t-\mu_t\mathbf 1}
{\max\bigl\{\|z_t-\mu_t\mathbf 1\|_2,\rho\bigr\}}.
\label{eq:normalized-sliding-window}
\end{equation}
The parameter $\rho$ removes the singularity at constant windows and gives
$\|y_t\|_2\leq1$.  Addition of a constant to $v$ leaves the cloud unchanged.

This is the time-series instance of \cref{sec:functional-viewpoint}: $\KOm$ is
the full simplicial complex on the window indices $t$, and the filtration is
the Vietoris--Rips filtration of the point cloud $\{y_t\}_t$,
\begin{equation}
\Fil_{\mathrm{SW}}(v):=\{\KOm_\tau(v)\}_{\tau\geq0},
\qquad
\KOm_\tau(v):=\bigl\{\sigma\in\KOm:\operatorname{diam}\{y_t\}_{t\in\sigma}\leq\tau\bigr\},
\label{eq:sw-filtration}
\end{equation}
Thus edges join windows at distance at most $\tau$, and a triangle or
higher simplex enters when all its vertices are pairwise within that
distance. Increasing $\tau$ first connects nearby windows and eventually
fills loops in the point cloud. We write $\PH_1^{\mathrm{SW}}(v)$ for its $1$-dimensional persistent
homology.  Unlike the two support-based cases, the filtration is
ascending in $\tau$, and it depends on $v$ through the positions of the
windows rather than through values attached to fixed cells.  Following the periodicity score in
\cite[\S7]{perea2015sliding}, define
\begin{equation}
\PerS_{M,h}(v):=
\begin{cases}
\displaystyle\frac1{\sqrt3}
\max_{(b,d)\in\PH_1^{\mathrm{SW}}(v)}(d-b),
&\PH_1^{\mathrm{SW}}(v)\neq\varnothing,\\[3pt]
0,&\PH_1^{\mathrm{SW}}(v)=\varnothing.
\end{cases}
\label{eq:periodicity-score}
\end{equation}
For the Euclidean unit circle the principal
$1$-class persists from zero to $\sqrt3$ and dense finite samples approach this
value \cite{perea2015sliding,adamaszek2017vietoris}.  Thus a score near one
indicates a one-cycle with persistence comparable to that of the unit circle.

For a target $a\in[0,1]$, set
\begin{equation}
\mathcal Q_{\mathrm{per}}(v;a)
:=\bigl(\PerS_{M,h}(v)-a\bigr)^2.
\label{eq:periodicity-penalty}
\end{equation}
The choices $a=1$ and $a=0$ target circle-scale persistence and the absence
of persistent one-cycles, respectively. 
For an unordered target multiset
$A=\{a_1,\ldots,a_r\}$, remove the arbitrary labelling of basis rows with
an optimal assignment, where $\mathfrak S_r$ denotes the set of permutations
of $\{1,\ldots,r\}$:
\begin{equation}
\mathcal Q_{\mathrm{per}}(v_1,\ldots,v_r;A)
:=\min_{\pi\in\mathfrak S_r}
\sum_{j=1}^r\bigl(\PerS_{M,h}(v_j)-a_{\pi(j)}\bigr)^2.
\label{eq:periodicity-assignment-loss}
\end{equation}
The
minimum is obtained by sorting the scores and targets in the same order.  
The target $A$ is a modelling choice.  In the
numerical experiments, targets equal to one designate rhythmic bases and
targets equal to zero designate trend or residual bases. The synthetic
example in \cref{subsec:periodic-main} illustrates this separation;
the diurnal-cycle study is detailed in \suppref{subsubsec:exp-diurnal}.

\section{Optimisation of \TopNMF}
\label{sec:optimization}

We describe projected first-order updates for the regularised objective,
explain how to differentiate the scores, and state the resulting
convergence properties and computational costs.

\subsection{A scale-fixed objective}
\label{subsec:scale-normalized-objective}

For every positive diagonal matrix $D$,
$(\W D)(D^{-1}\V)=\W\V$.  To compare scale-sensitive penalties across basis rows, we fix
the maximum of each non-zero row to one and work on
\begin{equation}
\mathcal S
:=
\{v\in\Rnn^d:\|v\|_\infty=1\},
\qquad
C
:=
\Rnn^{n\times r}\times\mathcal S^r.
\label{eq:scale-fixed-feasible-set}
\end{equation}
Thus every basis row
lies in $[0,1]^d$ and has maximum entry one.

For $(\W,\V)\in C$, define
\begin{equation}
\mathcal R(\V)
:=
\sum_{j\in\mathcal J_{\mathrm{top}}}
\mathcal Q_j(v_j),
\qquad
F(\W,\V):=\Lapx(\X,\W\V)+\lambdatop\mathcal R(\V),
\label{eq:normalized-topnmf-objective}
\end{equation}
where
\[
\mathcal Q_j\in
\left\{
\TP^{(k)},\ \WTP_p^{(k)},\
\CP_\alpha,\ \mathcal Q_{\mathrm{per}}(\,\cdot\,;a_j)
\right\},
\]
with $k\in\{0,1\}$, $p\in\Z_{\geq0}$, $\alpha>0$, and $a_j\in[0,1]$.
The permutation-invariant periodicity loss
\eqref{eq:periodicity-assignment-loss} may replace the corresponding rowwise
sum.  All listed terms except $\CP_\alpha$ are non-negative.  For a graph
$G=(U,E)$ and $u\in[0,1]^E$, \eqref{eq:cp-kruskal} gives the finite bound
\[
-\bigl(|U|-1\bigr)-\alpha\bigl(|E|-|U|+1\bigr)
\leq \CP_\alpha(u)\leq0,
\]
so the scale-fixed objective is bounded below.

\subsection{Projected first-order updates}
\label{subsec:projected-subgradient-updates}

Let $\operatorname{proj}_{\mathcal S^r}$ denote a 
Euclidean projection onto $\mathcal S^r$.  For one row $x$, it
clips $x$ to $[0,1]^d$ and sets a coordinate attaining $\max_i x_i$ to one,
breaking ties by index.  
A projected Clarke-subgradient step has the form
\begin{align}
\W^{(t+1)}
&=
\left[
\W^{(t)}
-\eta_t\,2(\W^{(t)}\V^{(t)}-\X)(\V^{(t)})^\top
\right]_+,
\label{eq:w-update-topnmf}
\\
\widehat{\V}^{(t+1)}
&=
\V^{(t)}
-\eta_t\left\{
2(\W^{(t)})^\top(\W^{(t)}\V^{(t)}-\X)
+\lambdatop G^{(t)}
\right\},
\label{eq:v-unprojected-update-topnmf}
\\
\V^{(t+1)}
&\in
\operatorname{proj}_{\mathcal S^r}
\bigl(\widehat{\V}^{(t+1)}\bigr),
\label{eq:v-update-topnmf}
\end{align}
where $G^{(t)}\in\partial^C\mathcal R(\V^{(t)})$,
$[\cdot]_+$ is entrywise projection onto the non-negative orthant, and
$\partial^C$ denotes the Clarke subdifferential, which extends the gradient
to locally Lipschitz functions at nonsmooth points
\cite{clarke1983optimization}.

On a region where the active filtration and score formulas are fixed,
the objective is smooth; such a region is called a smooth stratum. The
topological derivatives there have concrete combinatorial forms.  For support-based scores, fixing the persistence pairing identifies
each birth and death with a critical cell; differentiating the resulting
polynomial gives the stratum gradient
\cite{carriere2021optimizing,leygonie2022framework}.  For the graph score,
fixing a maximum-weight spanning tree $T$ gives
\begin{equation}
\frac{\partial\CP_\alpha}{\partial v_j(e)}
=
\begin{cases}
2(1-v_j(e)),&e\in T,\\
-2\alpha v_j(e),&e\notin T.
\end{cases}
\label{eq:cp-piecewise-gradient}
\end{equation}
Automatic differentiation computes these gradients on smooth strata.
At nonsmooth points, the convergence result below assumes
Clarke-subgradient selections.

\subsection{Regularity and convergence}
\label{subsec:regularity-convergence}

Each proposed score, and hence $F$, is locally Lipschitz and semialgebraic.
The first property controls the change in a score under small perturbations;
the second means that its graph can be described by finitely many polynomial
equalities and inequalities. The feasible set $C$ is closed and
semialgebraic as well. These properties allow the projected updates to be
analysed using the convergence framework of
\cite[Cor.~6.4]{davis2020stochastic}, also used for persistence objectives
in \cite{carriere2021optimizing}.

Specifically, suppose the iterates of
\eqref{eq:w-update-topnmf}--\eqref{eq:v-update-topnmf} are bounded, the
updates use $G^{(t)}\in\partial^C\mathcal R(\V^{(t)})$, and the positive
step sizes satisfy
\[
\sum_t\eta_t=\infty,\qquad \sum_t\eta_t^2<\infty.
\]
Then every accumulation point $z^\star$ satisfies
\begin{equation}
0\in\partial^C F(z^\star)+N_C(z^\star),
\label{eq:stationarity}
\end{equation}
where $N_C$ is the Clarke normal cone. This is the stationarity condition
for the constrained nonsmooth objective. The regularity argument and
verification of the convergence theorem's hypotheses are given in
\suppref{subsec:proof-optimization}.

\subsection{Computational cost and scaling}
\label{subsec:computational-cost}

For dense matrices, evaluating the reconstruction loss and its two gradients
costs $\mathcal O(ndr)$ operations per full-data iteration.  Projection onto
the non-negative orthant and Euclidean projection onto $\mathcal S^r$ cost $\mathcal O(r(n+d))$.
The additional cost of \TopNMF{} is the score and
gradient evaluation for each row in $\mathcal J_{\mathrm{top}}$:
\begin{itemize}
\item For a grid complex with $m=|\KOm|$ cells, the cost is that of persistent
      homology on the filtered complex.  Generic boundary-matrix reduction
      has worst-case cost $\mathcal O(m^3)$ in time and $\mathcal O(m^2)$ in
      memory, although cubical and dimension-specific implementations can be
      substantially faster
      \cite{dey2022computational,kaji2020cubical}.
\item For a graph basis with $d=|E|$ edges, Kruskal's algorithm evaluates the
      score and identifies the gradient in $\mathcal O(d\log d)$ time; the
      gradient itself is then assembled in $\mathcal O(d)$ time.
\item For a length-$d$ time series, the sliding-window cloud has
      $N=d-Mh$ points in $\R^{M+1}$.  Forming all pairwise distances costs
      $\mathcal O(N^2M)$.  A Vietoris--Rips complex truncated at dimension two
      can already contain $\mathcal O(N^3)$ simplices, so computing persistent homology of the
      window cloud is the principal scaling bottleneck, even with implementations
      that avoid building the complex explicitly \cite{bauer2021ripser}.
\end{itemize}
Thus the graph score can be evaluated directly from the edge weights,
whereas
Vietoris--Rips persistent homology is appropriate only for modest window clouds unless
one introduces subsampling or sparse-complex approximations.

\section{Numerical Validation}
\label{sec:numerical-validation}

Following the order of \cref{sec:data-type-topological-scores}, we examine
connected image parts, clique structure on graphs, and periodic components
in time series. Each example instantiates the basis-regularised objective
\eqref{eq:topnmf-objective}. We compare \TopNMF{} with standard NMF and
with basis sparsity ($\ell_1$) or domain-appropriate smoothness priors.
For \TopNMF{}, reported weights refer to reconstruction loss averaged over
matrix entries and the topological penalty averaged over basis rows
(\suppref{sec:experiments}).

When reference parts are available, \emph{atom recovery} is the mean cosine
similarity between optimally matched basis rows and reference parts; one
indicates agreement up to scale and ordering. The reference parts are used
only for evaluation. Reconstruction is measured by root mean squared
error (RMSE) over observed entries, after refitting non-negative
coefficients for each learned basis. Full settings and metrics are given
in \suppref{sec:experiments}, with sensitivity analyses in
\suppref{sec:supporting}.

\subsection{Grid-valued scalar fields: connected image parts}
\label{subsec:hangul-main}

The six binary $64\times64$ Hangul glyphs in \cref{fig:hangul-main} combine three consonants and two vowels.
These five reusable parts, called \emph{jamo}, give an exact rank-five
representation. To encourage each basis image to form one connected part,
we use $\mathcal Q_j=\TP^{(0)}$ from
\eqref{eq:total-squared-persistence} in the rowwise loss
\eqref{eq:topological-loss}, with $r=5$ and $\lambda_{\rm top}=0.1$.
As described in \cref{subsec:support-based-data}, the score penalises
persistent extra components of the superlevel sets. On the contractible
pixel grid, \cref{prop:ph-support-topology} identifies zero penalty with
connectedness of every non-empty superlevel set. This promotes spatial
coherence across thresholds without prescribing the shapes of the parts.

\begin{figure}[tb]
\centering
\includegraphics[width=0.92\linewidth]{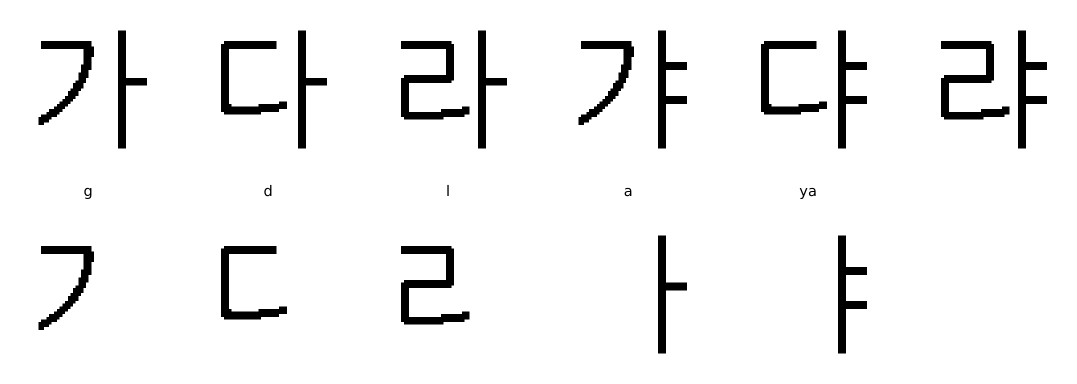}\\[6pt]
\includegraphics[width=\linewidth]{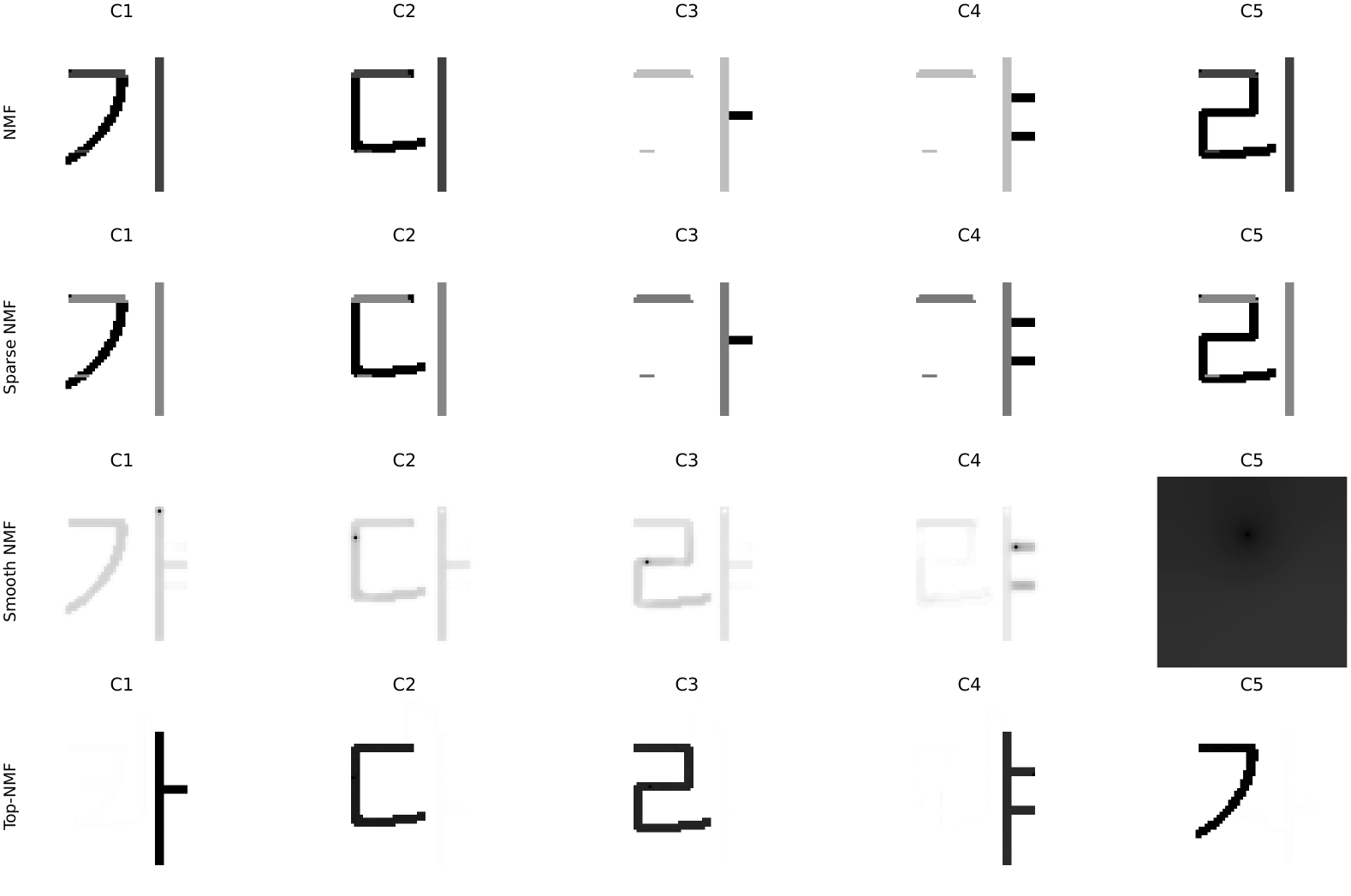}
\caption{Hangul glyphs and learned parts. Top: the six observed glyphs and
five segmented reference jamo. Below: bases learned by standard NMF,
sparse NMF, Smooth NMF, and \TopNMF{}. Basis order is arbitrary within each
method.}
\label{fig:hangul-main}
\end{figure}

\TopNMF{} recovers the
jamo almost exactly, while standard and sparse NMF retain mixtures
(\cref{fig:hangul-main}). Atom recovery is $0.9999$, compared with
$0.8093$ for standard NMF, while RMSE increases from $1.6\times10^{-7}$
to $0.00341$. Smooth NMF also gives connected components at threshold $0.3$,
but its blurred bases have lower recovery.
Full results are given in \suppref{subsubsec:exp-hangul}.
Further grid experiments use the weighted score
\eqref{eq:weighted-total-squared-persistence} for unimodal signals and
$\TP^{(0)}+\TP^{(1)}$ for connected, loop-free bar images
(\suppref{subsubsec:exp-1d-unimodal,subsubsec:exp-bars}).

\subsection{Graph data: clique components}
\label{subsec:graph-main}

Three observations mix overlapping unit-weight cliques of sizes five,
three, and four on a common set of nine vertices. Following
\cref{subsec:graph-data}, we use $\mathcal Q_j=\CP_1$ from
\eqref{eq:clique-promoting-functional}, with $r=3$ and
$\lambda_{\rm top}=0.01$. The score rewards persistent separation of
vertex groups and cycles within them; it is evaluated by the spanning-tree
formula \eqref{eq:cp-kruskal}. Each generating template is a cluster graph
with one non-trivial component of at least three vertices, so
\cref{thm:cp-local-minimizers} makes it a strict local minimiser of the
score. The experiment tests whether this preference helps recover the
cliques from their mixtures.

\begin{figure}[tb]
\centering
\includegraphics[width=\linewidth]{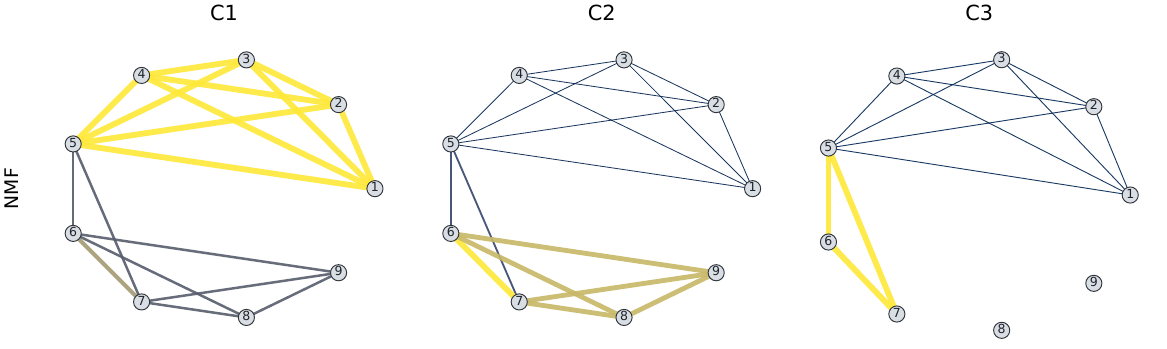}\\[4pt]
\includegraphics[width=\linewidth]{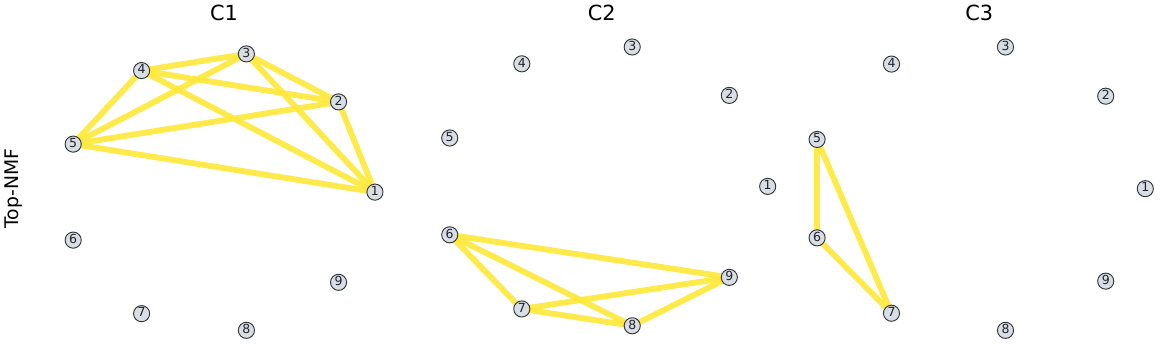}
\caption{Bases learned from mixtures of three overlapping cliques:
standard NMF (top) and \TopNMF{} (bottom). Edge colour and thickness encode
the basis weight; yellow denotes weight one.}
\label{fig:graph-main-results}
\end{figure}

Standard NMF and \TopNMF{} both reconstruct the observations with
negligible error, but the former retains graded mixtures of the templates,
whereas the latter recovers the three cliques
(\cref{fig:graph-main-results}). Sparse NMF also recovers them to the
reported precision. Graph-Laplacian smoothing does not separate them:
incident edges may belong to different overlapping cliques, whereas
$\CP_1$ responds to components and cycles throughout the edge filtration.
The full comparison is given in \suppref{subsubsec:exp-graph-synthetic}.

On the SocioPatterns school contact networks, \TopNMF{} increases the
within-class mass of the strongest basis edges from $0.966$ for standard
NMF to $0.985$, while RMSE changes from $0.0425$ to $0.0428$.
The comparison uses a class-balanced subset of students, with class
labels also determining the rank. Full settings and all baselines are
reported in \suppref{subsec:exp-graphs-sociopatterns}.

\subsection{Time series: periodic components}
\label{subsec:periodic-main}

Four length-$100$ observations mix a shifted cosine, a shifted triangle
wave, and a linear trend. After maximum normalisation, independent Gaussian
noise with standard deviation $0.01$ is added and negative values are clipped
to zero. The two oscillatory parts share period $\pi$
on the sampled interval $[0,2\pi]$. Here the relevant topology is that of
the sliding-window cloud in \cref{subsec:time-series-latent-topology}:
$\PerS_{M,h}$ in \eqref{eq:periodicity-score} measures its longest-lived
loop. With $r=3$, the permutation-invariant loss
\eqref{eq:periodicity-assignment-loss} and targets $A=\{0,1,1\}$ encourage
persistent loops in two basis rows and none in the third, without fixing
their order. We set $M=4$, $h=10$, and
$\lambda_{\rm top}=10^{-5}$. The known period determines the delay; the waveform shapes are learned
from the mixtures.

\begin{figure}[tb]
\centering
\includegraphics[width=0.90\linewidth]{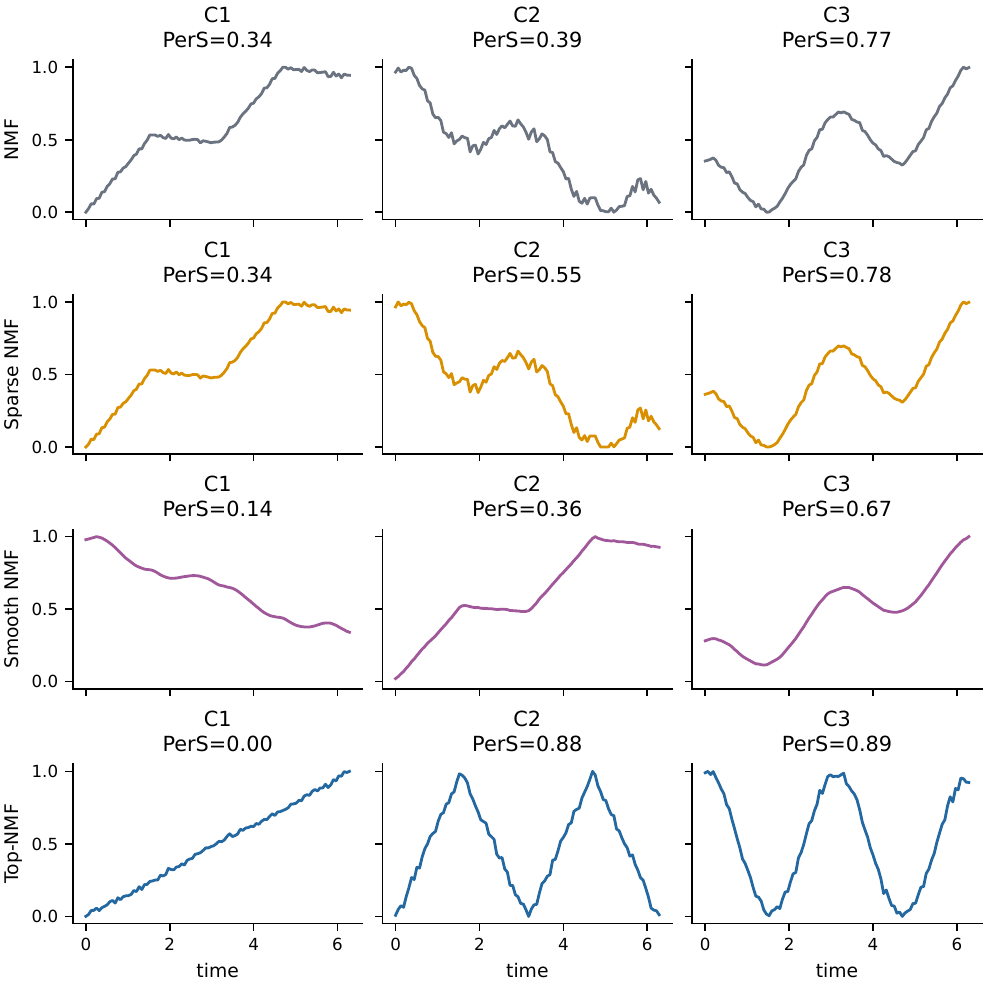}
\caption{Bases learned from noisy mixtures of two periodic waveforms and a
linear trend (seed $0$), sorted within each method by increasing periodicity score.
The displayed PerS values denote $\PerS_{4,10}$; \TopNMF{} uses the target
multiset $\{0,1,1\}$ and $\lambda_{\rm top}=10^{-5}$.}
\label{fig:periodic-main-results}
\end{figure}

\TopNMF{} separates the trend from the two oscillatory components while
retaining their different shapes (\cref{fig:periodic-main-results}).
Across three paired noise and initialisation seeds, mean atom recovery is
$0.9994$, compared with $0.8634$ for standard NMF; the corresponding
RMSEs on noisy observations are $0.00468$ and $0.00467$.
The periodicity prior improves recovery while keeping reconstruction
error close to that of standard NMF. The mixtures, full metrics,
and optimisation settings are given in
\suppref{subsubsec:exp-periodic-synthetic}.

For air-temperature data from six cities and five seeds, the most periodic
\TopNMF{} basis has a dominant $24$-hour Fourier period in $29$ of $30$
fits, compared with $19$ for standard NMF and $26$ for Smooth NMF.
Smooth NMF has the larger mean diurnal power share, while \TopNMF{} more
consistently identifies a basis with the target dominant period in these
fits. The data, settings, and Fourier diagnostics are detailed in
\suppref{subsubsec:exp-diurnal}.

\section{Conclusion}
\label{sec:conclusion}

\TopNMF{} provides a way to encode domain-specific topological preferences
in the basis functions learned by NMF. Its strengths are a common formulation
for connectivity on grids, clique-like graph structure, and periodicity in
time series, together with an exact spanning-tree formula and a characterisation
of local minimisers for the graph score. The regularity analysis supports
projected first-order optimisation. Across the examples studied, the priors
can improve structural diagnostics and atom recovery while keeping
reconstruction error explicit as a competing objective.

The method's scope is limited by the need to choose a filtration and score
that reflect the structure of interest, and to tune their balance with
reconstruction. The reported gains are empirical results on the studied
synthetic and real datasets; they do not establish improvements in downstream
recognition accuracy or a uniform advantage over all NMF methods. The
convergence analysis also does not provide a guarantee for the fixed-step
Adam procedure used in the numerical studies. These qualifications matter
when selecting a structural prior for a new application.

The framework may benefit structured representation learning when the desired
basis structure can be expressed through persistent homology. Future work
could extend the approach to other representation models, evaluate it on
additional application-specific tasks, and reduce the computational cost of
persistence calculations.

\ifdefined\arxivversion
\else
\submissionstatements
\fi

\ifdefined\arxivversion
  \clearpage
  \appendix
  \renewcommand{\thesection}{\Alph{section}}
  \crefalias{section}{appendix}
  \crefalias{subsection}{appendix}
  \crefalias{subsubsection}{appendix}
  \counterwithin{figure}{section}
  \counterwithin{table}{section}
  \section*{Appendices}
\section{Proofs and Further Analysis}
\label{sec:proofs}

We supply the proofs of the graph results and the optimisation analysis
stated in the main article. The graph score, admissible set $\mathcal A$,
and objective $F$ retain the definitions given there. References to these
results and their defining equations refer to the main article.

\subsection{Spanning-tree representation}
\label{subsec:proof-kruskal}

\begin{proof}[Proof of \cref{prop:kruskal-representation}]
Consider a Kruskal run producing $T(v)$.  Each
accepted edge merges two components and kills, at $v(e)$, a $0$-class born at
one.  Each rejected edge creates a $1$-cycle that persists to the end of the
filtration.  Omitting zero-length intervals gives
\eqref{eq:ph0-kruskal} and \eqref{eq:ph1-kruskal} since
\[
\begin{aligned}
\PH_0^{\mathrm{fin}}(v)
&=\{(1,v(e)):e\in T(v),\ v(e)<1\},\\
\PH_1^{\mathrm{tr}}(v)
&=\{(v(e),0):e\notin T(v),\ v(e)>0\}.
\end{aligned}
\]
For \eqref{eq:cp-tree-envelope}, write
$q_T(v)=-\alpha\sum_{e\in E}v(e)^2+\sum_{e\in T}h(v(e))$,
where $h(x)=\alpha x^2-(1-x)^2$ is strictly increasing on $[0,1]$.
Thus $q_T(v)$ is maximised by precisely the MWSTs of $v$,
as
strictly increasing transformations of the weights preserve
the MWST family.
\end{proof}

\subsection{Betti-number representation and binary graphs}
\label{subsec:proof-cp-binary}

\Cref{prop:kruskal-representation} has a tree-free restatement which exhibits
$\CP_\alpha$ as the graph instance of the passage from threshold-wise Betti
numbers to a continuous score described in \cref{sec:support-to-ph}.  
Write
$\beta_i(\tau):=\operatorname{rank}H_i(G_\tau(v))$.  For any MWST $T(v)$ and any
$\tau\in(0,1]$ the prefix $\{e\in T(v):v(e)\geq\tau\}$ is a maximum spanning
forest of $G_\tau(v)$, whence
\[
\beta_0(\tau)-1=\#\{e\in T(v):v(e)<\tau\},
\qquad
\beta_1(\tau)=\#\{e\notin T(v):v(e)\geq\tau\}.
\]
Substituting $(1-v(e))^2=2\int_{v(e)}^1(1-\tau)\,d\tau$ and
$v(e)^2=2\int_0^{v(e)}\tau\,d\tau$ into \eqref{eq:cp-kruskal} and exchanging sum
and integral gives
\begin{equation}
\CP_\alpha(v)
=-2\int_0^1\Bigl[(1-\tau)\bigl(\beta_0(\tau)-1\bigr)
                +\alpha\,\tau\,\beta_1(\tau)\Bigr]\,d\tau .
\label{eq:cp-betti-integral}
\end{equation}
Components
remaining separate down to low thresholds contribute to the first term while
cycles born at high thresholds contribute to the second.

\begin{proof}[Proof of \cref{cor:cp-binary}]
For $v=\mathbf 1_{E_1}$ the graphs $G_\tau(v)$ coincide with $(U,E_1)$ for every
$\tau\in(0,1]$, so $\beta_0(\tau)=c$ and
$\beta_1(\tau)=|E_1|-|U|+c$ are constant there.  Now evaluate
\eqref{eq:cp-betti-integral} using $2\int_0^1(1-\tau)\,d\tau=2\int_0^1\tau\,d\tau=1$.
\end{proof}

\subsection{Local minimisers}
\label{subsec:proof-cp-local}

\begin{proof}[Proof of \cref{thm:cp-local-minimizers}]
Suppose first that $v$ is a local minimiser.  If a weight
$x\in(0,1)$ occurs, varying all edges of weight $x$ by the same small
amount preserves their ordering and admissibility.  A MWST
remains unchanged, so \eqref{eq:cp-kruskal} restricts to a strictly concave
quadratic along this segment, contradicting local minimality.
Hence $v=\mathbf 1_{E_1}$.

If a component of $(U,E_1)$ is not complete, raise a missing edge within it
from zero to $\varepsilon\in(0,1)$.  Kruskal's algorithm still rejects that edge, since
its endpoints are joined by a unit-weight path; all other weights are
unchanged.  Equation \eqref{eq:cp-kruskal} shows that the score decreases by
$\alpha\varepsilon^2$.  Thus $(U,E_1)$ is a cluster graph.
If it has a two-vertex component with edge $e$ and another unit edge exists,
lowering $v(e)$ to $1-\varepsilon$ preserves admissibility.  Kruskal's algorithm still
accepts $e$, and the score decreases by $\varepsilon^2$.
This proves necessity.

Conversely, let $v=\mathbf 1_{E_1}$ satisfy the stated conditions and let
$v+w\in\mathcal A$, with $t:=\|w\|_\infty\in(0,1/2)$.
Then $w\leq0$ on $E_1$ and $w\geq0$ elsewhere.
Choose an edge $e$ with $|w(e)|=t$, and an MWST $T$ of $v$ that contains
$e$ if $e\notin E_1$, and omits $e$ if $e\in E_1$.
Such a tree exists: a zero-weight edge joins distinct cliques and can be
used to connect their spanning trees; a unit-weight edge in a clique of
size at least three can be omitted from its spanning tree.
When $E_1$ is a single edge, admissibility forces its weight to remain one,
so only the first case can occur.

Using $q_T(v)=\CP_\alpha(v)$ in \eqref{eq:cp-tree-envelope} gives
\begin{align*}
\CP_\alpha(v+w)-\CP_\alpha(v)
&\geq q_T(v+w)-q_T(v)\\
&=2\sum_{f\in T\setminus E_1}w(f)
  +2\alpha\sum_{f\in E_1\setminus T}(-w(f))\\
&\qquad-\sum_{f\in T}w(f)^2
  -\alpha\sum_{f\notin T}w(f)^2\\
&\geq2\min\{1,\alpha\}t-|E|\max\{1,\alpha\}t^2.
\end{align*}
Indeed, all linear terms are non-negative and the chosen edge contributes
at least $2\min\{1,\alpha\}t$.  This proves \eqref{eq:cp-local-growth},
whose right-hand side is positive for all sufficiently small $t>0$.
\end{proof}

\subsection{Global minimisers}
\label{subsec:proof-cp-global}

\begin{proof}[Proof of \cref{cor:cp-global-minimizers}]
Continuity and compactness of $\mathcal A$ give a global minimiser, which is
a unit-weight cluster graph by \cref{thm:cp-local-minimizers}.
Let its component sizes be $s_1,\ldots,s_c$ and set
$m=N-c=\sum_i(s_i-1)$.  Since at least one edge is present, $1\leq m\leq N-1$.
By \eqref{eq:cp-binary},
\[
-\CP_\alpha(v)
=N-m-1+\frac\alpha2\sum_{i=1}^c(s_i-1)(s_i-2)
\leq N-m-1+\frac\alpha2m(m-1).
\]
The last expression is strictly convex in $m$, so its maximum on
$[1,N-1]$ is attained only at an endpoint.  The endpoint values are
$N-2$ and $\alpha(N-1)(N-2)/2$, attained respectively by a single edge
and the complete graph.  These are also the only cluster graphs with
$m=1$ and $m=N-1$, proving the claims.
\end{proof}

\subsection{Regularity and projected-subgradient convergence}
\label{subsec:proof-optimization}

The support-based and graph scores are locally Lipschitz and semialgebraic
by the pairing argument in \cref{subsec:support-based-data} and the finite
maximum representation \eqref{eq:cp-tree-envelope}. For sliding windows,
the positive floor $\rho$ makes the normalisation Lipschitz; pairwise
distances and their finite maxima are Lipschitz as well. As a function of
the cell filtration values, the longest finite lifetime is $2$-Lipschitz
in the $\ell_\infty$ norm: on each ordering region it is the maximum of
zero and the paired death-minus-birth differences, with formulas agreeing
at ties. All maps involved are semialgebraic. Squaring target errors and
taking a finite minimum over assignments preserve both semialgebraicity
and local Lipschitz continuity. Thus every proposed score, and hence $F$,
has these properties; the feasible set $C$ is closed and semialgebraic.
This is the finite-stratum approach to persistence objectives used in
\cite{carriere2021optimizing,leygonie2022framework}.

The updates \eqref{eq:w-update-topnmf}--\eqref{eq:v-update-topnmf} are the
deterministic case of the proximal stochastic subgradient method in
\cite[Cor.~6.4]{davis2020stochastic}, with $f=F$, $g=0$, and $X=C$.
The proximal map is a measurable selection of projection onto $C$.
For one basis row, $\mathcal S$ is the union of the cube faces
$\{u\in[0,1]^d:u_i=1\}$. Clipping to the cube and setting a coordinate
attaining the largest input value to one selects a nearest face and hence
a nearest point in $\mathcal S$. Breaking ties by index makes this a
fixed measurable rule, even when the nearest point is not unique.

Assume the boundedness, step-size, and Clarke-subgradient selection
conditions stated in \cref{subsec:regularity-convergence}. Local
boundedness of the Clarke subdifferential follows from local Lipschitz
continuity; the oracle is deterministic, and $g\equiv0$. These observations
verify the oracle conditions in Assumption~E of
\cite{davis2020stochastic}. Semialgebraicity of $F$ and $C$ then allows
Corollary~6.4 to be applied, yielding the stationarity condition
\eqref{eq:stationarity} at every accumulation point.

\section{Numerical Experiments}
\label{sec:experiments}

This section gives the experimental settings and detailed results for
\cref{sec:numerical-validation} of the main article.  The studies compare
basis structure and reconstruction accuracy under the priors in
\cref{tab:experiment-summary}.

\ifdefined\arxivversion
\else
\supplementdataavailability
\fi

\begin{table}[bhp]
\centering
\small
\renewcommand{\arraystretch}{1.12}
\begin{tabular}{@{}>{\raggedright\arraybackslash}p{0.16\linewidth}
                    >{\raggedright\arraybackslash}p{0.27\linewidth}
                    >{\raggedright\arraybackslash}p{0.50\linewidth}@{}}
\toprule
Domain & Interpretability prior & Topological score\\
\midrule
1-D signals & Unimodal peaks &
$\WTP_2^{(0)}$ on line superlevel sets
(\Cref{subsec:support-based-data})\\
Images & Connected spatial parts &
$\TP^{(k)}$, $k\in\{0,1\}$, on cubical superlevel sets
(\Cref{subsec:support-based-data})\\
Graphs & Clique-like dense components &
$\CP_\alpha$ on descending edge filtrations
(\Cref{subsec:graph-data})\\
Time series & Separation of rhythmic and trend components &
$(\PerS_{M,h}-a)^2$ on normalised sliding-window embeddings
(\Cref{subsec:time-series-latent-topology})\\
\bottomrule
\end{tabular}
\caption{Interpretability priors and corresponding topological regularisers
used in the supplementary experiments.}
\label{tab:experiment-summary}
\end{table}

\paragraph{Common implementation}
We initialise NMF with NNDSVDar
\cite{boutsidis2008nndsvd} unless otherwise stated.  \TopNMF{} uses Adam with learning rate $0.005$.
Unless stated otherwise, each step projects coefficients onto the
non-negative orthant and basis rows in Euclidean distance onto
$\{v\geq0:\|v\|_\infty=1\}$, as described in \cref{sec:optimization}.
The reported \TopNMF{} weights refer to the averaged objective
\begin{equation}
\frac{1}{nd}\Lapx(\X,\W\V)
+\lambdatop\frac{1}{r}\Ltop(\V).
\label{eq:experimental-objective}
\end{equation}
For the experiments here, which regularise all $r$ rows, the equivalent
weight in the main article's unaveraged loss \eqref{eq:topnmf-objective} is
$(nd/r)\lambdatop$.

Cubical persistent homology is computed with Cubical Ripser
\cite{kaji2020cubical}, Vietoris--Rips persistent homology with GUDHI
\cite{gudhi:urm}.

Standard and sparse NMF use scikit-learn coordinate descent
\cite{scikit-learn}, with at most $3{,}000$ iterations and tolerance
$10^{-6}$.  Sparse NMF penalises the basis with $\alpha_H=10^{-3}$,
$\alpha_W=0$, and \texttt{l1\_ratio}$=1$.
The quadratic-smoothness baseline, denoted Smooth NMF, uses
$r^{-1}\sum_j\sum_{\{p,q\}\in E_{\rm grid}}(v_j(p)-v_j(q))^2$ on
nearest-neighbour grid edges.  This quadratic penalty differs from the
usual total-variation seminorm \cite{yin2010nonnegative}.
For graph bases, feature-side graph-Laplacian NMF \cite{cai2010graph}
uses the line graph of the edge features.  Both smoothness baselines use
the same projected Adam engine as \TopNMF{}.

\paragraph{Common evaluation}
For $\X\in\R^{n\times d}$,
\[
\operatorname{RMSE}(\X,\W\V)=\frac{\|\X-\W\V\|_F}{\sqrt{nd}}.
\]
\emph{Atom recovery} is the mean cosine similarity to the reference atoms
under optimal one-to-one assignment.  
Before evaluation, basis rows are maximum-normalised and non-negative
coefficients are refitted by non-negative least squares (NNLS) on the
observed matrix.  RMSE is computed using these refitted coefficients.
Support Betti numbers are averaged
across rows at threshold $0.3$ of each row maximum
(\cref{eq:upper-star-superlevel}).  Sparsity is the mean Hoyer score
$(\sqrt d-\|v_j\|_1/\|v_j\|_2)/(\sqrt d-1)$ over non-zero rows
\cite{hoyer2004non}. 
Unless stated otherwise, tables report means over seeds $0,1,2$.
\subsection{Vector and image bases: support topology on grids}
\label{subsec:exp-support-signals}

These studies penalise the superlevel-set topology of basis functions on
line and pixel grids.

\subsubsection{Unimodal synthetic signals}
\label{subsubsec:exp-1d-unimodal}

\paragraph{Data}
We start with four Gaussian atoms on the line grid
$\{1,\ldots,100\}$.
The data comprise six non-negative mixtures.  On the rescaled
coordinate $x\in[0,1]$, the atom centres are $0.18$, $0.40$, $0.62$, and
$0.85$, each with standard deviation $0.05$.  Mixture coefficients are
sampled uniformly from $[0.4,1.2]$, and independent half-normal noise
$0.01|Z|$, $Z\sim N(0,1)$, is added (\cref{fig:signal-unimodal-data}).

\begin{figure}[tb]
\centering
\includegraphics[width=0.82\linewidth]{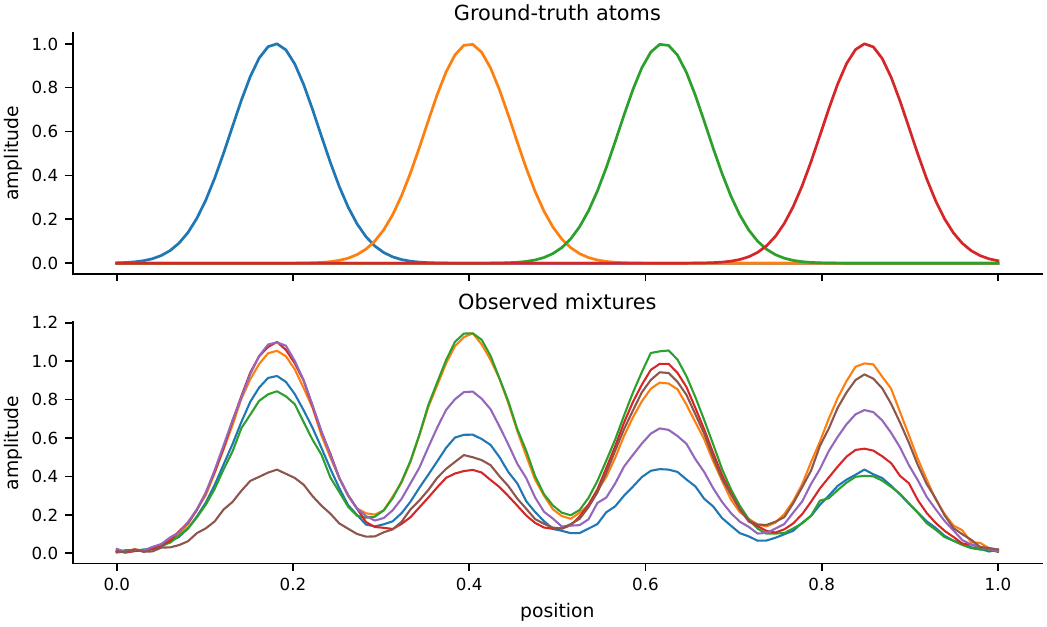}
\caption{Ground-truth peaks and the six observed mixtures.}
\label{fig:signal-unimodal-data}
\end{figure}

\paragraph{Model and evaluation}
We seek one peak per row using rank $r=4$,
$\Ltop(\V)=\sum_{j=1}^{4}\WTP_2^{(0)}(v_j)$,
$\lambda_{\rm top}=0.5$, and $3{,}000$ Adam steps.
Smooth NMF uses weight $0.02$.
A detected mode is a local maximum of height at least $0.10$ and prominence
at least $0.05$ after row normalisation.  Peak MAE is the mean absolute
distance between row maxima and generating centres in the rescaled
coordinate, minimised over one-to-one assignments.

\begin{figure}[tb]
\centering
\includegraphics[width=0.92\linewidth]{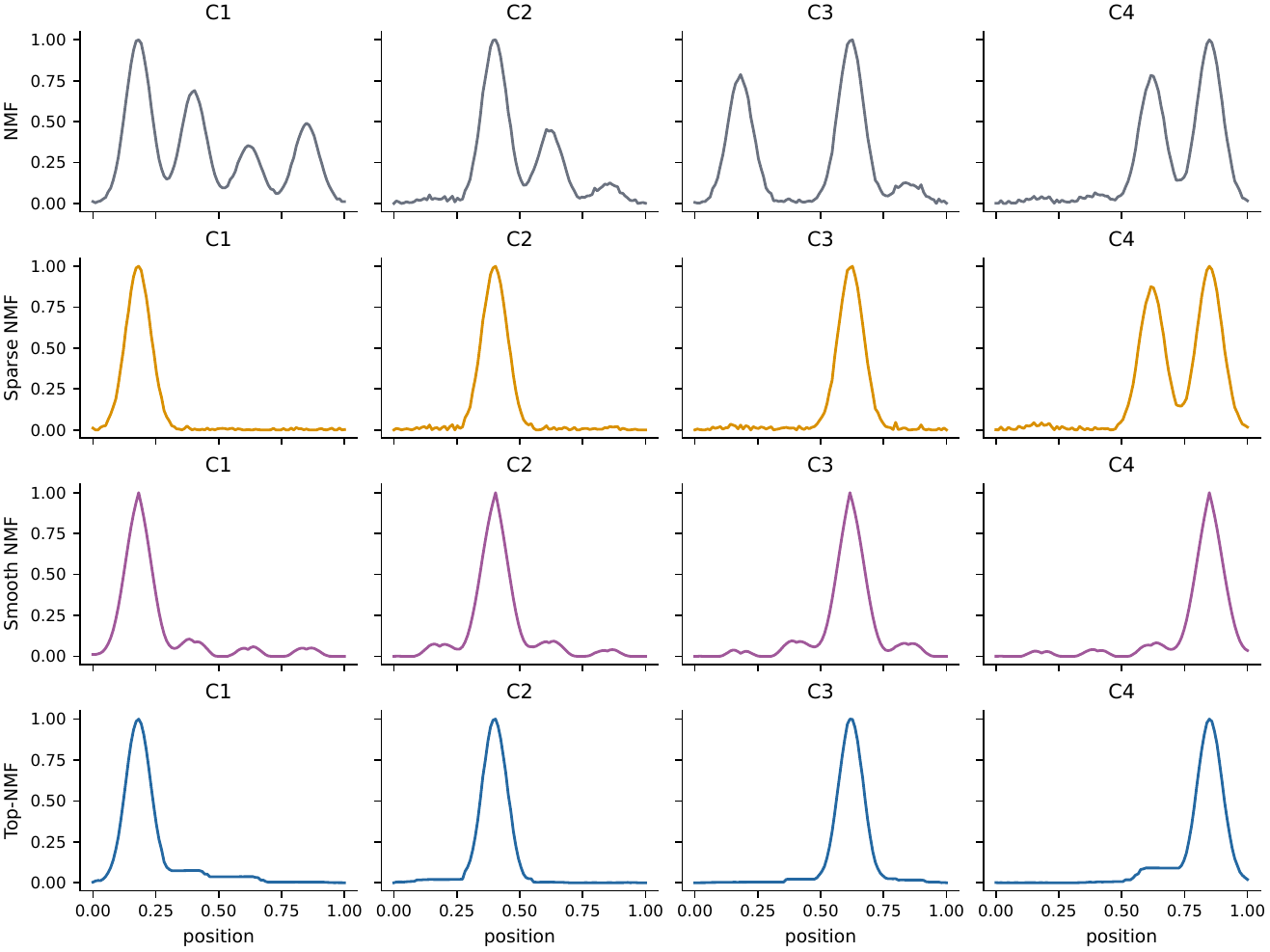}
\caption{Learned bases for the unimodal-signal experiment.}
\label{fig:signal-unimodal-bases}
\end{figure}

\begin{table}[htbp]
\centering
\small
\begin{tabular}{@{}lrrrrr@{}}
\toprule
Method & RMSE & Atom rec. & Peak MAE & Modes & $\WTP_2^{(0)}$\\
\midrule
NMF & $0.0036809$ & $0.80457$ & $0.0034091$ & $3$ & $0.36386$\\
Sparse NMF & $0.0038653$ & $0.93801$ & $0.0034091$ & $1.25$ & $0.10163$\\
Smooth NMF & $0.010592$ & $0.99099$ & $0.002803$ & $1.25$ & $0.0055019$\\
\TopNMF{} & $0.0064961$ & $0.99658$ & $0.002803$ & $1$ & $4.7\times10^{-6}$\\
\addlinespace[3pt]
Generating atoms & $0.007463$ & $1$ & $0.002803$ & $1$ & $0$\\
\bottomrule
\end{tabular}
\caption{Unimodal-signal metrics. }
\label{tab:signal-unimodal-metrics}
\end{table}

\paragraph{Findings}
\Cref{tab:signal-unimodal-metrics,fig:signal-unimodal-bases} show that
\TopNMF{} most closely recovers the generating peaks and is the only method
with one detected mode in every component.  Sparsity and quadratic-smoothness regularisation
also improve recovery over standard NMF, but retain secondary peaks.
\TopNMF{} has higher reconstruction error than standard and sparse NMF,
and lower error than Smooth NMF.

\subsubsection{Synthetic bar images}
\label{subsubsec:exp-bars}

\paragraph{Data}
The six binary atoms on a $12\times12$ grid are three horizontal and three
vertical bars.  Each has $\beta_0=1$ and $\beta_1=0$.  The 30 observations are
sums of three to five distinct atoms with coefficients sampled uniformly from
$[0.7,1.0]$, followed by clipping to $[0,1]$
(\cref{fig:bars-atoms,fig:bars-data}).

\begin{figure}[hbt]
\centering
\includegraphics[width=0.72\linewidth]{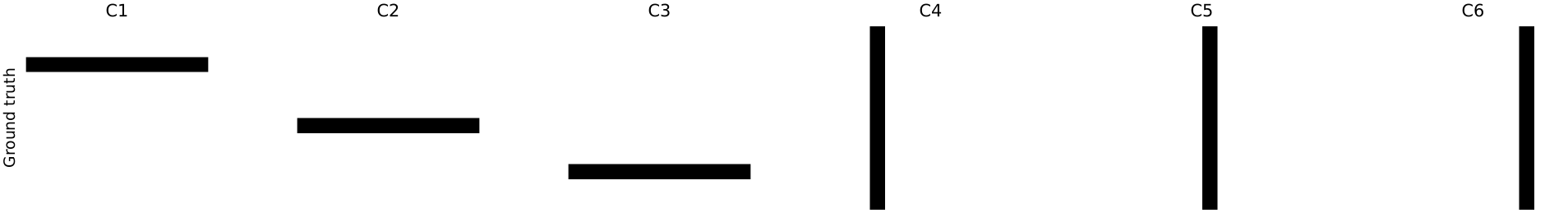}
\caption{The six ground-truth bar atoms.}
\label{fig:bars-atoms}
\end{figure}

\begin{figure}[hbt]
\centering
\includegraphics[width=0.82\linewidth]{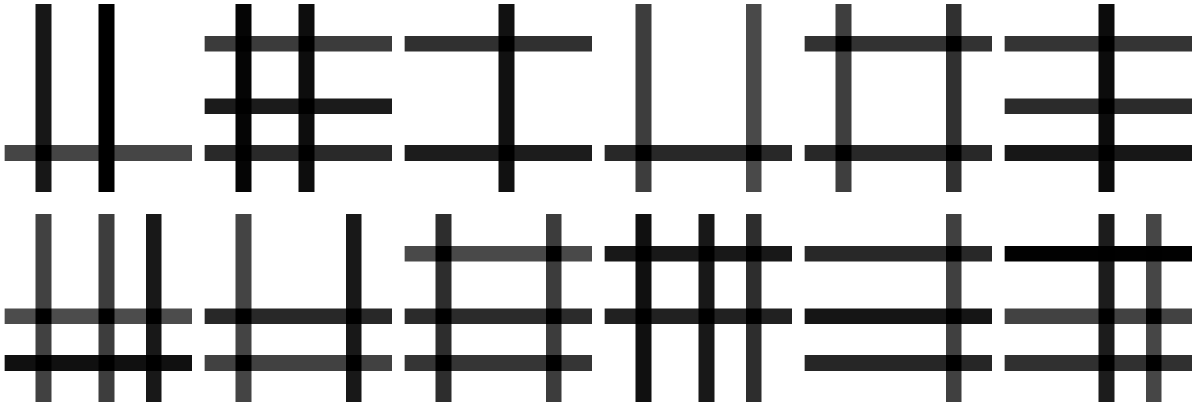}
\caption{Twelve of the 30 synthetic bar mixtures.}
\label{fig:bars-data}
\end{figure}

\paragraph{Model setting}
We use rank $r=6$ and
\[
\Ltop(\V)=\sum_{j=1}^{6}\sum_{k=0}^{1}\TP^{(k)}(v_j),
\qquad \lambda_{\rm top}=0.5,
\]
for $5{,}000$ Adam steps.  The comparison set is standard NMF, sparse NMF,
and two-dimensional Smooth NMF with weight $0.002$.

\begin{figure}[tb]
\centering
\includegraphics[width=0.96\linewidth]{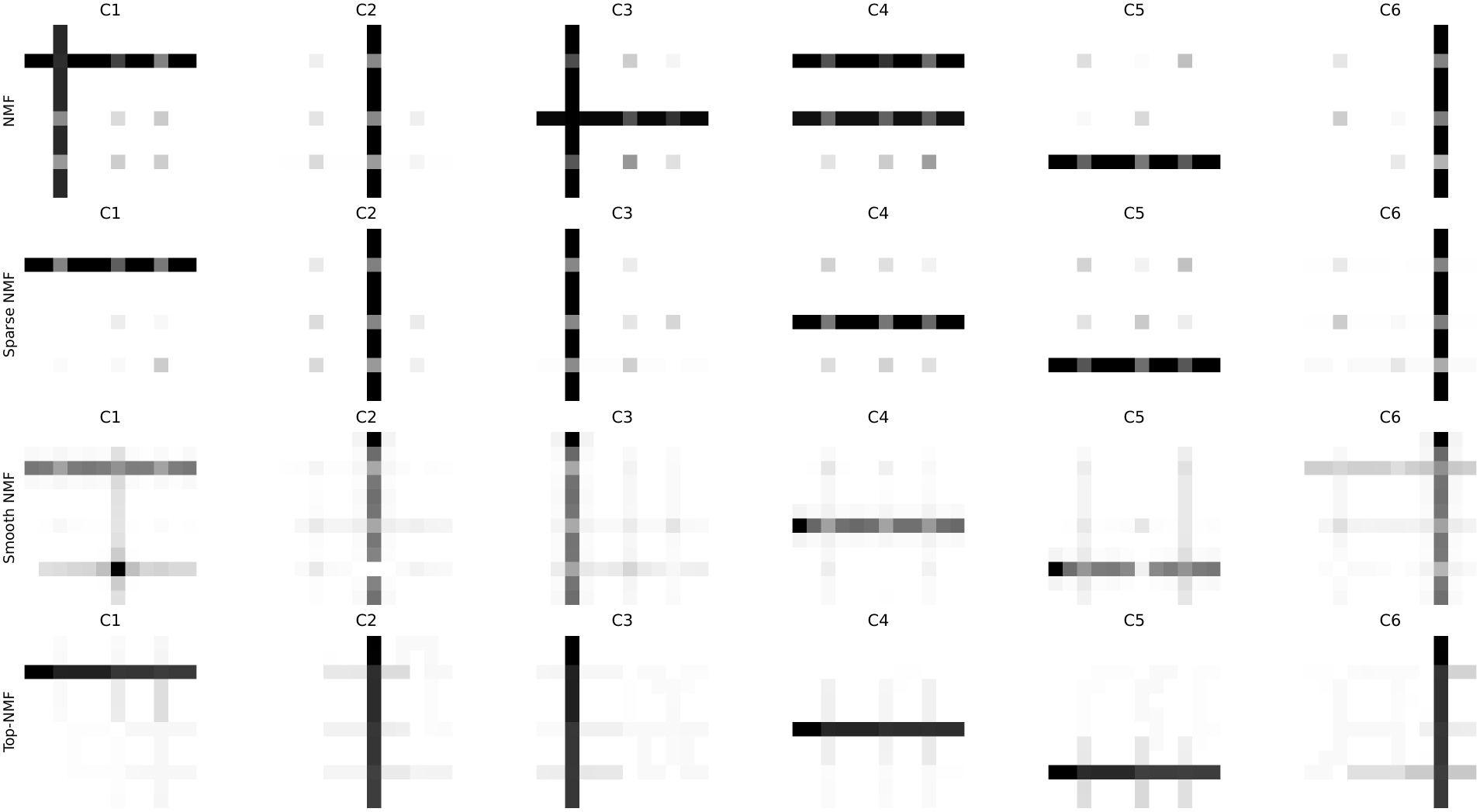}
\caption{Learned bases for the synthetic bar experiment.}
\label{fig:bars-bases}
\end{figure}

\begin{table}[htbp]
\centering
\small
\begin{tabular}{@{}lrrrr@{}}
\toprule
Method & RMSE & Atom rec. & $\beta_0$ & $\beta_1$\\
\midrule
NMF & $0.067454$ & $0.84136$ & $1.6667$ & $0$\\
Sparse NMF & $0.037409$ & $0.96747$ & $1$ & $0$\\
Smooth NMF & $0.097093$ & $0.90359$ & $1.6667$ & $0$\\
\TopNMF{} & $0.078663$ & $0.98668$ & $1$ & $0$\\
\addlinespace[3pt]
Generating atoms & $0.09946$ & $1$ & $1$ & $0$\\
\bottomrule
\end{tabular}
\caption{Synthetic-bar metrics.}
\label{tab:bars-metrics}
\end{table}

\paragraph{Findings}
\TopNMF{} most closely recovers the bars, while sparse NMF reconstructs the
mixtures more accurately (\cref{tab:bars-metrics,fig:bars-bases}).
Both yield connected, loop-free thresholded components at threshold $0.3$.

\subsubsection{Hangul glyphs: settings and metrics}
\label{subsubsec:exp-hangul}

\Cref{subsec:hangul-main,fig:hangul-main} of the main article present the
data, learned bases, and their relation to the connectivity score.
The six glyphs are \emph{ga}, \emph{da}, \emph{ra}, \emph{gya},
\emph{dya}, and \emph{rya}. Splitting each binary $64\times64$ raster at
column $41$ within a shared blank strip gives five distinct reference
jamo. Each glyph is the sum of one consonant and one vowel with disjoint
supports, so these parts reproduce the data exactly at rank five.

\paragraph{Model setting}
We use rank $r=5$, $\Ltop(\V)=\sum_{j=1}^{5}\TP^{(0)}(v_j)$,
$\lambda_{\rm top}=0.1$, and $5{,}000$ Adam steps.  Standard NMF, sparse NMF,
Smooth NMF, and \TopNMF{} share deterministic NNDSVDa initialisation, so
only one fit is reported.  The two-dimensional quadratic-smoothness weight is $0.002$.

\begin{table}[htbp]
\centering
\small
\begin{tabular}{@{}lrrrrr@{}}
\toprule
Method & RMSE & Atom rec. & $\beta_0$ & $\beta_1$ & Sparsity\\
\midrule
NMF & $1.6\times10^{-7}$ & $0.80931$ & $1.8$ & $0$ & $0.76171$\\
Sparse NMF & $6.8\times10^{-5}$ & $0.87191$ & $2.4$ & $0$ & $0.75652$\\
Smooth NMF & $0.09014$ & $0.6495$ & $1$ & $0$ & $0.57308$\\
\TopNMF{} & $0.0034051$ & $0.99991$ & $1$ & $0$ & $0.79311$\\
\addlinespace[3pt]
Manual segmentation & $0$ & $1$ & $1$ & $0$ & n/a\\
\bottomrule
\end{tabular}
\caption{Hangul-glyph metrics.}
\label{tab:hangul-metrics}
\end{table}

\Cref{tab:hangul-metrics} reports the complete comparison. Betti numbers
are averaged over basis rows at threshold $0.3$.
Atom recovery is measured against the manually segmented
jamo, which reproduce the six glyphs exactly.

\subsection{Graph bases}
\label{subsec:exp-graphs}

Each observation assigns non-negative weights to the edges of a fixed
complete graph.  Both studies report \emph{clique quality}
$-r^{-1}\sum_j\CP_1(v_j)$ with $\alpha=1$, evaluated on every method's learned basis.
Larger values indicate a lower graph penalty (\cref{cor:cp-binary}).

\subsubsection{Synthetic clique mixtures}
\label{subsubsec:exp-graph-synthetic}

\paragraph{Data}
On the complete graph with nine vertices, the ground-truth templates are the
cliques on $(1,2,3,4,5)$, $(5,6,7)$, and $(6,7,8,9)$; the first two share
vertex $5$ and the last two share $6$ and $7$.  The complete dataset contains
the three mixtures with coefficient vectors $(3,1,1)$, $(1,1,3)$, and
$(1,3,1)$, each maximum-normalised, so that every observation emphasises one
template (\cref{fig:graph-data,fig:graph-atoms}).

\begin{figure}[hbt]
\centering
\includegraphics[width=0.94\linewidth]{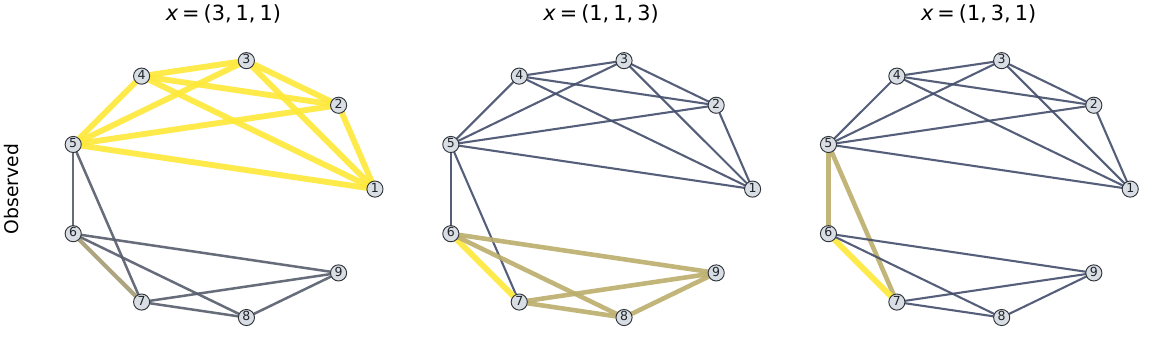}
\caption{All three observations in the synthetic clique experiment.}
\label{fig:graph-data}
\end{figure}

\begin{figure}[tb]
\centering
\includegraphics[width=0.94\linewidth]{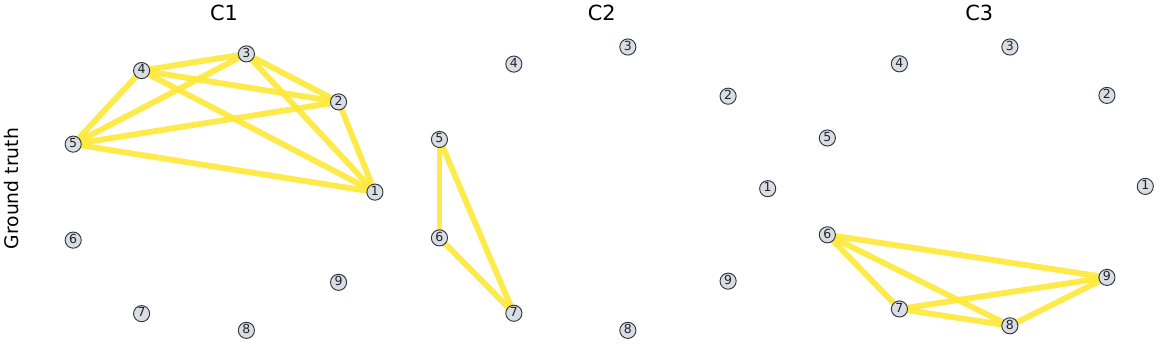}
\caption{The three overlapping ground-truth clique templates.}
\label{fig:graph-atoms}
\end{figure}

\paragraph{Model setting}
We use rank $r=3$, $\alpha=1$,
$\Ltop(\V)=\sum_{j=1}^{3}\CP_\alpha(v_j)$,
$\lambda_{\rm top}=0.01$, and $5{,}000$ Adam steps.  We compare with standard
NMF, sparse NMF, and graph-Laplacian NMF.  The Laplacian is built on the line
graph of the 36 edge features and has weight $0.02$.

\begin{figure}[tb]
\centering
\begin{minipage}[t]{0.48\linewidth}
\includegraphics[width=\linewidth]{Figures/graph_synthetic_basis_nmf.pdf}
\end{minipage}\hfill
\begin{minipage}[t]{0.48\linewidth}
\includegraphics[width=\linewidth]{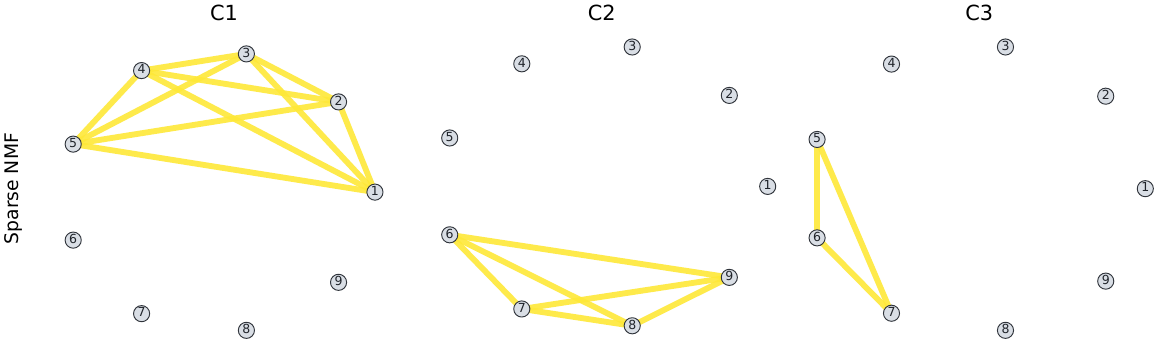}
\end{minipage}

\begin{minipage}[t]{0.48\linewidth}
\includegraphics[width=\linewidth]{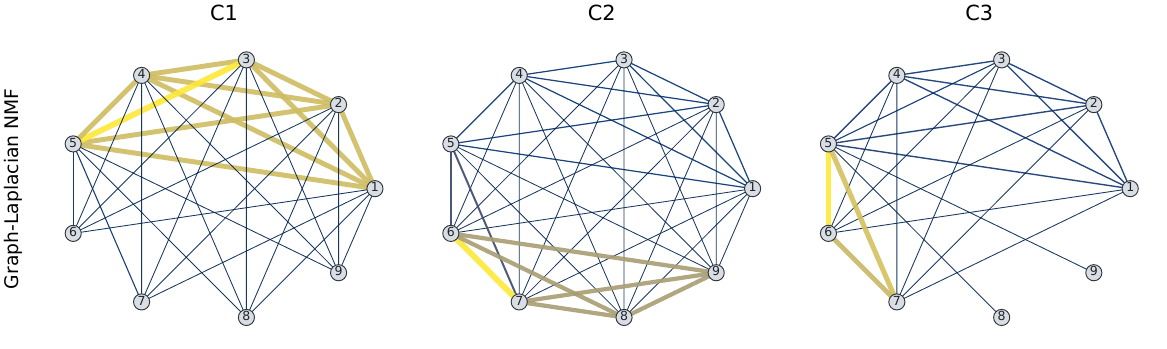}
\end{minipage}\hfill
\begin{minipage}[t]{0.48\linewidth}
\includegraphics[width=\linewidth]{Figures/graph_synthetic_basis_top_nmf.pdf}
\end{minipage}
\caption{Synthetic graph bases: NMF, sparse NMF, graph-Laplacian NMF, and
\TopNMF{} (reading order).  The first and last panels also appear as \Cref{fig:graph-main-results} in the main text.}
\label{fig:graph-supp-results}
\end{figure}

\begin{table}[htbp]
\centering
\small
\begin{tabular}{@{}lrrrr@{}}
\toprule
Method & RMSE & Atom rec. & Clique qual. & Sparsity\\
\midrule
NMF & $1.9\times10^{-8}$ & $0.97689$ & $7.1478$ & $0.64577$\\
Sparse NMF & $1.3\times10^{-5}$ & $1$ & $8.333$ & $0.71041$\\
Graph-Lap.\ NMF & $0.01934$ & $0.97898$ & $6.4969$ & $0.62817$\\
\TopNMF{} & $7.8\times10^{-17}$ & $1$ & $8.3333$ & $0.71041$\\
\addlinespace[3pt]
Generating atoms & $7.8\times10^{-17}$ & $1$ & $8.3333$ & $0.71041$\\
\bottomrule
\end{tabular}
\caption{Synthetic clique metrics.}
\label{tab:graph-synthetic-metrics}
\end{table}

\paragraph{Findings}
\TopNMF{} and sparse NMF recover the generating cliques, with \TopNMF{}
also matching their penalty and achieving a numerically exact
reconstruction (\cref{tab:graph-synthetic-metrics,fig:graph-supp-results}).
Standard NMF nearly reconstructs the mixtures but mixes the templates.
Laplacian smoothing also fails to separate them: incident edges can belong
to different overlapping cliques. 

\subsubsection{SocioPatterns contact graphs}
\label{subsec:exp-graphs-sociopatterns}

\paragraph{Data}
We use the 2012 SocioPatterns high-school deployment
\cite{fournet2014contact}, which contains face-to-face proximity events at
20-second resolution and class labels for 180 students in five classes.  The
five most active students in every class are retained.
Each observation is a 30-minute contact-duration graph advanced in 15-minute
steps; windows with fewer than four selected-subgraph contacts are discarded
and each retained row is maximum-normalised.  The resulting data matrix contains
250 windows on the 300 edges of a 25-vertex complete graph. Class labels
are used to select five students per class, set the rank, and evaluate
within-class edge mass; they are not used in the factorisation objective.

\paragraph{Model setting}
The rank equals the number of classes.  \TopNMF{} uses $\alpha=1$,
$\lambda_{\rm top}=10^{-4}$, and 600 Adam steps.  Graph-Laplacian NMF uses the
edge-feature line graph, weight $0.02$, and the same step budget.  Standard
and sparse NMF use the common scikit-learn settings.

\paragraph{Evaluation}
In addition to clique quality we report a label-based diagnostic.
For basis $v_j$, let $E_j(s)$ contain the edges at or above the quantile of order $s$
of its positive weights.  If $c(u)$ is the class of vertex $u$, define
\[
M_j(s)=
\frac{\sum_{e=\{u,u'\}\in E_j(s)}v_j(e)\mathbf1\{c(u)=c(u')\}}
     {\sum_{e\in E_j(s)}v_j(e)}.
\]
For non-zero rows, the denominator is positive.  We report the component
mean of $M_j(0.70)$, termed within-class top-edge mass.

\begin{figure}[tbp]
\centering
\includegraphics[width=\linewidth]{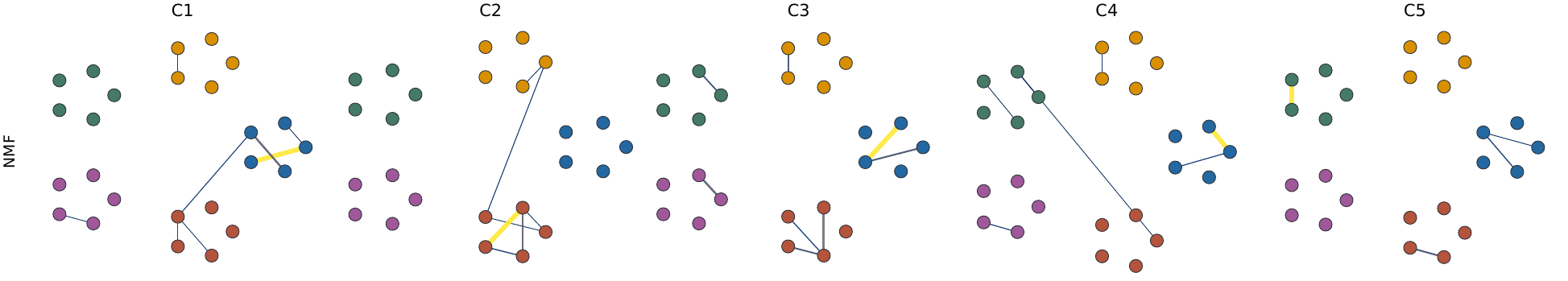}\\[5pt]
\includegraphics[width=\linewidth]{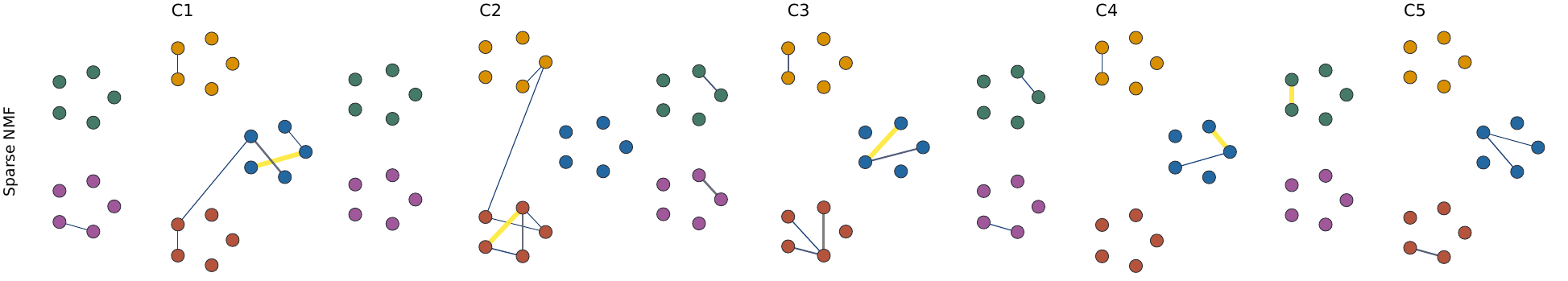}\\[5pt]
\includegraphics[width=\linewidth]{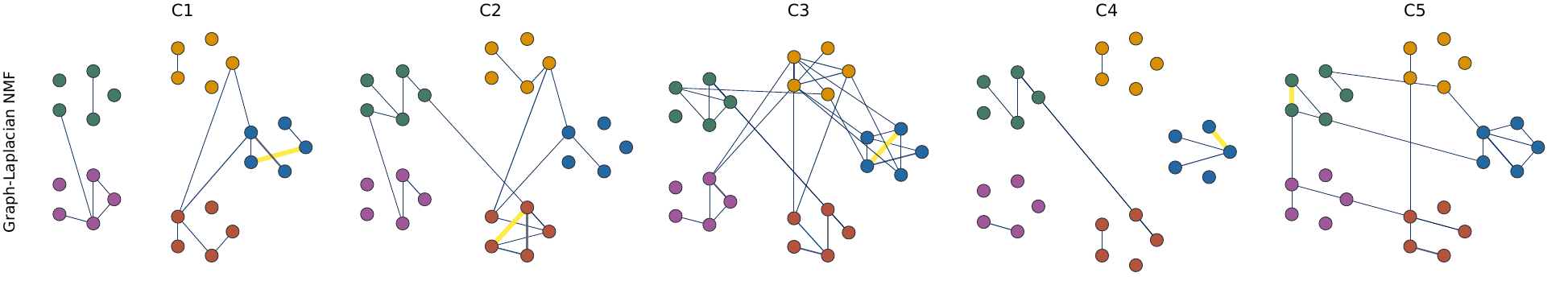}\\[5pt]
\includegraphics[width=\linewidth]{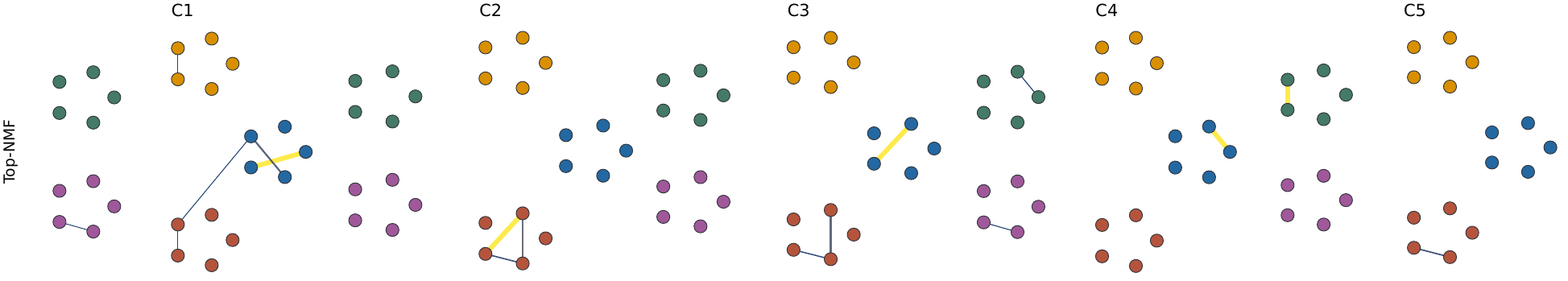}
\caption{SocioPatterns 2012 bases, from top to bottom: standard NMF, sparse NMF,
graph-Laplacian NMF, and \TopNMF{}.  Vertices are coloured by class; edge colour and thickness
encode the basis weight.}
\label{fig:sociopatterns-2012-bases}
\end{figure}

\begin{table}[htbp]
\centering
\small
\begin{tabular}{@{}lrrrr@{}}
\toprule
Method & RMSE & Within-class mass & Clique qual. & Sparsity\\
\midrule
NMF & $0.042468$ & $0.96555$ & $21.463$ & $0.95291$\\
Sparse NMF & $0.042468$ & $0.96632$ & $21.471$ & $0.95328$\\
Graph-Lap.\ NMF & $0.042502$ & $0.93942$ & $21.631$ & $0.95257$\\
\TopNMF{} & $0.04281$ & $0.98522$ & $22.291$ & $0.97676$\\
\bottomrule
\end{tabular}
\caption{SocioPatterns 2012 metrics.}
\label{tab:sociopatterns-metrics}
\end{table}

\paragraph{Findings}
\TopNMF{} concentrates more weight within classes and has the highest clique
quality, with a small increase in reconstruction error
(\cref{tab:sociopatterns-metrics,fig:sociopatterns-2012-bases}).

\subsection{Time-series bases}
\label{subsec:exp-time-series}

These studies test whether persistent homology of sliding-window embeddings separates oscillatory
components from trends.
For sorted basis scores $s_{(1)}\leq\cdots\leq s_{(r)}$ and $1\leq k<r$
targets equal to one, with the remaining targets zero, the \emph{periodicity gap} is
\[
\frac1k\sum_{j>r-k}s_{(j)}-\frac1{r-k}\sum_{j\leq r-k}s_{(j)}.
\]
The \emph{target loss} is $r^{-1}\sum_j(s_{(j)}-a_{(j)})^2$
(\cref{eq:periodicity-assignment-loss}).  Both diagnostics are evaluated
identically for every method.  The normalisation floor is $\rho=10^{-12}$
in both studies (\cref{eq:normalized-sliding-window}).

\subsubsection{Synthetic periodic components}
\label{subsubsec:exp-periodic-synthetic}

\Cref{subsec:periodic-main,fig:periodic-main-results} of the main article
present the learned bases and their relation to the sliding-window score.
Here we give the data construction, full settings, and numerical metrics.

\paragraph{Data}
At the sample points $t_i=2\pi(i-1)/99$, we define three atoms: a shifted cosine
$u_1(t)=\cos(2t)+1$, a shifted triangle wave
$u_2(t)=\operatorname{tri}(2t)+1$, and the trend $u_3(t)=t$, for
$i=1,\ldots,100$.  Here $\operatorname{tri}$ is the $2\pi$-periodic triangle
wave that equals $-1+2\theta/\pi$ on $[0,\pi]$ and $3-2\theta/\pi$ on
$[\pi,2\pi]$.  The four clean mixtures are
\[
u_1+u_3,\quad
u_1+0.4u_2+0.3u_3,\quad
0.9u_1+1.2u_2+0.3u_3,\quad
u_2+u_3.
\]
After normalising each mixture by its maximum, we denote the resulting
matrix by $\X_{\rm clean}$ and use noisy observations
\[
X^{(s)}_{ij}=\max\{(X_{\rm clean})_{ij}+0.01 Z^{(s)}_{ij},0\},
\qquad Z^{(s)}_{ij}\overset{\mathrm{iid}}{\sim}\mathcal N(0,1).
\]
The atoms and all four noisy mixtures for seed $0$ appear in
\cref{fig:time-series-data}.

\begin{figure}[tb]
\centering
\includegraphics[width=0.92\linewidth]{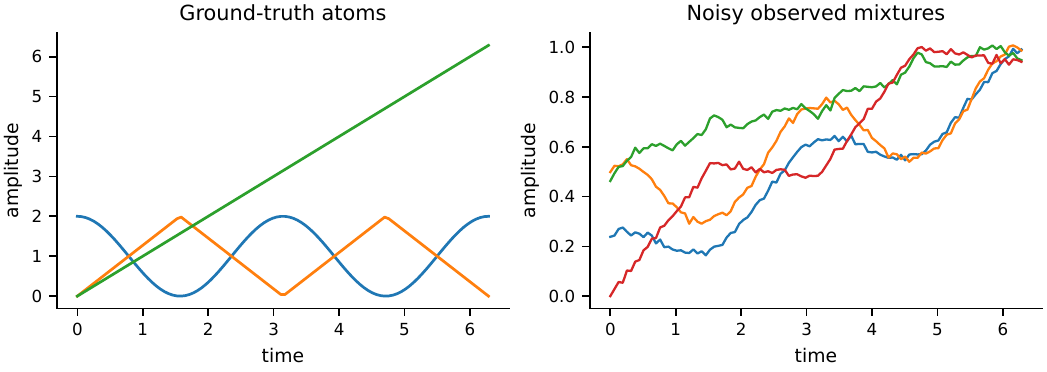}
\caption{Atoms and noisy observations in the synthetic periodicity experiment (seed $0$).}
\label{fig:time-series-data}
\end{figure}

\paragraph{Model setting}
We use rank $r=3$ and the permutation-invariant target loss
\[
\Ltop(\V)=
\min_{\pi\in\mathfrak S_3}\sum_{j=1}^{3}
\bigl(\PerS_{M,h}(v_j)-a_{\pi(j)}\bigr)^2,
\qquad A=\{0,1,1\},
\]
with $M=4$, $h=10$, $\lambda_{\rm top}=10^{-5}$, and $5{,}000$ Adam
steps.  The embedding span follows
$Mh\approx M/(M+1)\times\text{period}$
\cite[\S2.1 and \S5.2]{perea2015sliding}.
Smooth NMF uses weight $0.02$.

\begin{table}[htbp]
\centering
\small
\begin{tabular}{@{}lrrrr@{}}
\toprule
Method & RMSE & Atom rec. & Gap & Target loss\\
\midrule
NMF & $0.00467$ & $0.86335$ & $0.36736$ & $0.17659$\\
Sparse NMF & $0.00473$ & $0.87161$ & $0.31361$ & $0.15006$\\
Smooth NMF & $0.01025$ & $0.87905$ & $0.39251$ & $0.17430$\\
\TopNMF{} & $0.00468$ & $0.99942$ & $0.87517$ & $0.00952$\\
\addlinespace[3pt]
Generating atoms & $0.00957$ & $1.00000$ & $0.93356$ & $0.00298$\\
\bottomrule
\end{tabular}
\caption{Synthetic periodicity metrics.}
\label{tab:time-series-metrics}
\end{table}

\paragraph{Findings}
\Cref{tab:time-series-metrics} complements the basis comparison in
\cref{fig:periodic-main-results} of the main article. \TopNMF{} has the
highest periodicity gap and smallest target loss among the learned bases.
It also has the highest atom recovery, with a reconstruction error close
to those of standard and sparse NMF. The generating-atom reference has non-zero
residual because the observations contain noise.

\subsubsection{Diurnal cycle in air temperature}
\label{subsubsec:exp-diurnal}

\paragraph{Data and model}
We use hourly air temperature for
2019--2023 in Tokyo, Fukuoka, Sapporo, Phoenix, London, and Singapore,
from Open-Meteo \cite{openmeteo}.
Two-hourly temperature means form seven-day windows
($d=84$) at a one-day stride; each window is shifted and scaled to $[0,1]$.

\TopNMF{} uses rank $r=4$, targets $\{0,0,0,1\}$, $M=3$, $h=3$,
$\lambda_{\rm top}=10^{-3}$, and $800$ Adam steps.  The embedding follows
the same period-based heuristic as in the synthetic study, using a diurnal target period.  Each fit uses
$300$ randomly selected windows for five different seeds.  
In this study, the projected Adam methods (\TopNMF{} and Smooth NMF)
constrain each basis row to the cube $[0,1]^d$; maximum normalisation is
applied for evaluation. Smooth NMF uses weight $0.02$. Standard and sparse
NMF allow $4{,}000$ coordinate-descent iterations.

\paragraph{Evaluation}
We record whether the dominant Fourier period is diurnal and the fraction
of one-sided non-DC power at the diurnal frequency and its harmonics.
The numerator uses bins $7,14,\ldots,42$ and the denominator bins
$1,\ldots,42$.

\begin{table}[htbp]
\centering
\small
\begin{tabular}{@{}lrrr@{}}
\toprule
Method & Dominant & Diurnal share & Gap\\
\midrule
NMF & $19/30$ & $0.58\pm0.30$ &  $0.224$\\
Sparse NMF & $19/30$ & $0.56\pm0.29$ &  $0.224$\\
Smooth NMF & $26/30$ & $0.73\pm0.26$  & $0.278$\\
\TopNMF{} & $29/30$ & $0.69\pm0.19$  & $0.656$\\
\bottomrule
\end{tabular}
\caption{Diurnal-cycle checks over six cities.  \emph{Dominant}
counts the fits whose most periodic basis has a $24$-hour dominant Fourier
period; \emph{diurnal share} is the mean fraction of that basis's spectral energy
at the diurnal frequency and its harmonics, using one-sided non-DC Fourier
power, with its standard deviation across the fits.
The gap is computed from the same basis periodicity scores used by the
regulariser.  The Fourier diagnostics use the basis row with the highest
periodicity score.}
\label{tab:weather-diurnal-metrics}
\end{table}

\begin{figure}[tb]
\centering
\includegraphics[width=\linewidth]{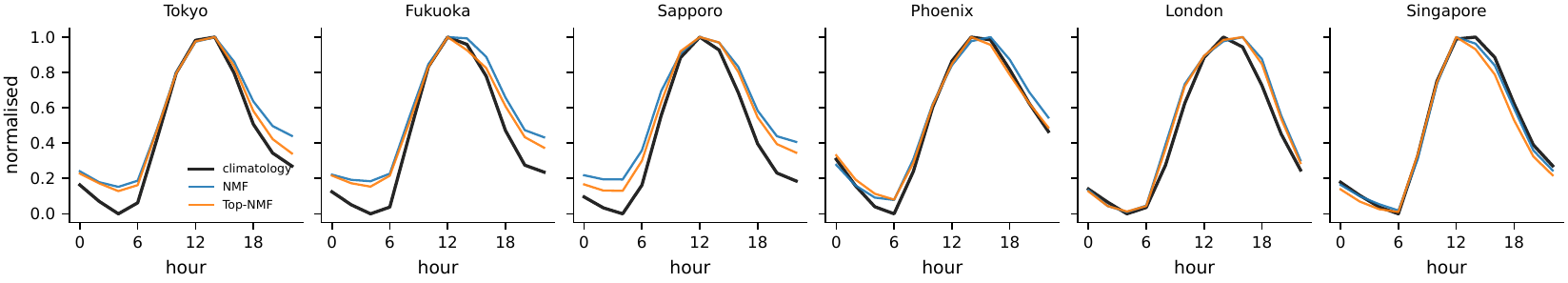}
\caption{Folded diurnal profile of the most periodic basis for standard
NMF and \TopNMF{}, against the climatological daily cycle of each city.}
\label{fig:weather-diurnal-templates}
\end{figure}

\paragraph{Findings}
\Cref{tab:weather-diurnal-metrics} shows that \TopNMF{} most often selects
a basis with a diurnal dominant period and has the smallest spread in
diurnal power share.  Smooth NMF has a slightly larger mean power share,
so periodic structure can also be promoted by smoothness in this example.

\section{Sensitivity Analyses}
\label{sec:supporting}

We examine whether the basis structures in the synthetic experiments
persist when the regularisation weight changes, and whether improving the
topological score also improves recovery of the generating atoms.
These are different questions: the regulariser expresses a structural
prior, whereas atom recovery tests whether that prior selects the intended
components of a particular mixture. We therefore report structural
scores together with atom recovery and reconstruction error.

Within each experiment, we vary only $\lambda_{\rm top}$ and retain the
data, rank, initialisation scheme, and optimisation settings specified in
\cref{sec:experiments}, except where stated below. Weights use the averaged loss in \eqref{eq:experimental-objective}; their
numerical values depend on the chosen score.
Each fit starts afresh with NNDSVDar initialisation using seeds $0,1,2$.
The unimodal and bar-image
sweeps use $3{,}000$ and $5{,}000$ Adam steps, respectively. The graph sweep
uses $1{,}000$ steps, compared with $5{,}000$ in the main comparison.
In the sensitivity tables, \emph{Max. s.d.} denotes, for each metric, the
largest sample standard deviation across the three seeds among the tested
weights.

\subsection{Sensitivity to the regularisation weight}
\label{par:weight-sensitivity}

\paragraph{Unimodal signals}
\Cref{tab:weight-unimodal} shows that the mean detected mode count remains
one at all five tested weights from $0.02$ to $0.5$. The penalty
falls from $7.7\times10^{-5}$ at $0.02$ to about $4\text{--}5\times10^{-6}$
at $0.2$ and $0.5$. Increasing the weight from $0.05$ to the main comparison's
$0.5$ raises RMSE from $0.00503$ to $0.00650$ and lowers atom recovery
from $0.99948$ to $0.99658$. Thus this example supports recovery over a
range of weights, while also showing that nearly eliminating the
persistence penalty can sacrifice fidelity without improving the
specified structural diagnostic. 

\begin{table}[tbp]
\centering
\small
\begin{tabular}{@{}lrrr@{}}
\toprule
$\lambda_{\rm top}$ & RMSE $\downarrow$ & Atom rec. $\uparrow$ & $\WTP_2^{(0)}\downarrow$\\
\midrule
$0.02$ & $0.00487$ & $0.99944$ & $7.7\times10^{-5}$\\
$0.05$ & $0.00503$ & $0.99948$ & $2.0\times10^{-5}$\\
$0.1$ & $0.00511$ & $0.99939$ & $7.8\times10^{-6}$\\
$0.2$ & $0.00527$ & $0.99896$ & $4.3\times10^{-6}$\\
$0.5$ & $0.00650$ & $0.99658$ & $4.7\times10^{-6}$\\
\addlinespace
Max. s.d. & $3.9\times10^{-5}$ & $2.9\times10^{-5}$ & $2.2\times10^{-6}$\\
\bottomrule
\end{tabular}
\caption{Unimodal-signal weight sweep, $3{,}000$ steps per fit. The score is averaged over basis rows. Entries are means over initialisation seeds $0,1,2$ on a fixed dataset.}
\label{tab:weight-unimodal}
\end{table}

\paragraph{Bar images}
Most tested weights (with the only exception at $0.01$) produce connected, loop-free supports at threshold
$0.3$ (\cref{tab:weight-bars}). Among those settings, the smallest weight
$0.002$ gives the lowest RMSE, $0.04145$, whereas $0.02$ gives the highest
atom recovery, $0.98968$. At the main comparison's weight $0.5$, recovery
remains high ($0.98668$), but RMSE rises to $0.07866$.

\begin{table}[tbp]
\centering
\small
\begin{tabular}{@{}lrrr@{}}
\toprule
$\lambda_{\rm top}$ & RMSE $\downarrow$ & Atom rec. $\uparrow$ & $\beta_0$\\
\midrule
$0.002$ & $0.04145$ & $0.97992$ & $1$\\
$0.005$ & $0.04780$ & $0.98551$ & $1$\\
$0.01$ & $0.08465$ & $0.84089$ & $1.16667$\\
$0.02$ & $0.06163$ & $0.98968$ & $1$\\
$0.05$ & $0.06975$ & $0.98873$ & $1$\\
$0.1$ & $0.07405$ & $0.98758$ & $1$\\
$0.2$ & $0.07675$ & $0.98680$ & $1$\\
$0.5$ & $0.07866$ & $0.98668$ & $1$\\
\addlinespace
Max. s.d. & $3.1\times10^{-4}$ & $9.3\times10^{-4}$ & $0$\\
\bottomrule
\end{tabular}
\caption{Bar-image weight sweep, $5{,}000$ steps per fit. $\beta_0$ is averaged over rows at threshold $0.3$; $\beta_1=0$ throughout. The desired value is $\beta_0=1$. 
}
\label{tab:weight-bars}
\end{table}

\paragraph{Clique mixtures}
At weights $0.005$, $0.01$, and $0.05$, the learned bases recover the
cliques to numerical precision, with unit atom recovery and RMSE below
$10^{-12}$ (\cref{tab:weight-graph}). The main comparison's weight $0.01$
therefore lies among several successful tested settings, even with the
shorter optimisation budget. At $0.001$, recovery is still $0.99829$,
but a small residual remains after $1{,}000$ steps.

At the largest tested weight, $0.1$, clique quality increases from
$8.333$ to $13.556$, while atom recovery falls to $0.779$ and RMSE rises
to $0.269$. This illustrates over-regularisation: the balance with
reconstruction is needed to retain the generating clique structure.

\begin{table}[tbp]
\centering
\small
\begin{tabular}{@{}lrrr@{}}
\toprule
$\lambda_{\rm top}$ & RMSE $\downarrow$ & Atom rec. $\uparrow$ & Clique quality $\uparrow$\\
\midrule
$0.001$ & $0.00143$ & $0.99829$ & $8.07180$\\
$0.005$ & $<10^{-12}$ & $1$ & $8.33333$\\
$0.01$ & $<10^{-12}$ & $1$ & $8.33333$\\
$0.05$ & $<10^{-12}$ & $1$ & $8.33333$\\
$0.1$ & $0.26917$ & $0.77887$ & $13.55553$\\
\addlinespace
Max. s.d. & $0.00312$ & $0.00172$ & $0.19246$\\
\bottomrule
\end{tabular}
\caption{Clique-mixture weight sweep, $1{,}000$ steps per fit. Clique quality is $-r^{-1}\sum_j\CP_1(v_j)$.
}
\label{tab:weight-graph}
\end{table}

\paragraph{Periodicity}
\label{par:periodicity-weight}
We test $\lambda_{\rm top}\in\{10^{-5},10^{-3},10^{-2},10^{-1},1\}$, keeping
$r=3$, targets $\{0,1,1\}$, $M=4$, $h=10$, and $5{,}000$ Adam steps.
We use the same three noisy datasets as in
\cref{subsubsec:exp-periodic-synthetic}, sharing each dataset and its
initialisation seed across all weights.

The periodicity gap is the mean of the two highest basis scores minus the
lowest; the target loss is their mean squared deviation from sorted
targets $(0,1,1)$, as defined in \cref{subsec:exp-time-series}.
A larger gap and a smaller target loss indicate the intended separation
in score space. Atom recovery checks whether this separation also
corresponds to the generating waveforms.

The main comparison uses $\lambda_{\rm top}=10^{-5}$. Among the five
tested weights, this setting gives the highest mean atom recovery
($0.99942$) and periodicity gap, and the smallest reconstruction error
and target loss (\cref{tab:periodicity-weight-noisy}). Increasing the weight
raises reconstruction error and reduces atom recovery, while the gap
and target loss vary non-monotonically. Stronger regularisation therefore
does not improve the structural diagnostics in this experiment.

\begin{table}[tbp]
\centering
\small
\begin{tabular}{@{}lrrrr@{}}
\toprule
$\lambda_{\rm top}$ & Gap $\uparrow$ & Target loss $\downarrow$ & RMSE $\downarrow$ & Atom rec. $\uparrow$\\
\midrule
$10^{-5}$ & $0.87517$ & $0.00952$ & $0.00468$ & $0.99942$\\
$10^{-3}$ & $0.81232$ & $0.01781$ & $0.00573$ & $0.97456$\\
$10^{-2}$ & $0.83784$ & $0.01777$ & $0.01595$ & $0.96272$\\
$10^{-1}$ & $0.85765$ & $0.01488$ & $0.02956$ & $0.92995$\\
$1$ & $0.84419$ & $0.01610$ & $0.04029$ & $0.91715$\\
\addlinespace
Max. s.d. & $0.11326$ & $0.01721$ & $0.00388$ & $0.02428$\\
\bottomrule
\end{tabular}
\caption{Periodic mixtures with noise standard deviation $0.01$. Entries are means over three paired noise and initialisation seeds, shared across weights.}
\label{tab:periodicity-weight-noisy}
\end{table}

\subsection{Implications for choosing the weight}

The sweeps show how the regularisation weight balances structural
agreement, atom recovery, and reconstruction error. Their interaction
depends on the data and prior: connected supports persist over a range of
weights for the grid examples, while excessive clique regularisation
favours structures that fit the observations poorly.
In practice, weights should be assessed jointly with reconstruction, an
application-relevant structural diagnostic, and stability across
initialisations. Atom recovery supplies such a diagnostic in controlled
synthetic tests; real applications require an external measure relevant
to the intended interpretation.

\fi

\bibliographystyle{elsarticle-num}
\bibliography{reference}

\begin{thebibliography}{10}
\expandafter\ifx\csname url\endcsname\relax
  \def\url#1{\texttt{#1}}\fi
\expandafter\ifx\csname urlprefix\endcsname\relax\def\urlprefix{URL }\fi
\expandafter\ifx\csname href\endcsname\relax
  \def\href#1#2{#2} \def\path#1{#1}\fi

\bibitem{mallat1999wavelet}
S.~Mallat, A Wavelet Tour of Signal Processing, 2nd Edition, Academic Press,
  San Diego, CA, 1999.

\bibitem{jolliffe2002principal}
I.~T. Jolliffe, Principal Component Analysis, 2nd Edition, Springer Series in
  Statistics, Springer, 2002.
\newblock \href {https://doi.org/10.1007/b98835} {\path{doi:10.1007/b98835}}.

\bibitem{lee1999learning}
D.~D. Lee, H.~S. Seung, Learning the parts of objects by non-negative matrix
  factorization, Nature 401~(6755) (1999) 788--791.
\newblock \href {https://doi.org/10.1038/44565} {\path{doi:10.1038/44565}}.

\bibitem{mairal2010online}
J.~Mairal, F.~Bach, J.~Ponce, G.~Sapiro,
  \href{https://jmlr.org/papers/v11/mairal10a.html}{Online learning for matrix
  factorization and sparse coding}, J. Mach. Learn. Res. 11~(2) (2010) 19--60.
\newline\urlprefix\url{https://jmlr.org/papers/v11/mairal10a.html}

\bibitem{hoyer2004non}
P.~O. Hoyer, \href{https://www.jmlr.org/papers/v5/hoyer04a.html}{Non-negative
  matrix factorization with sparseness constraints}, J. Mach. Learn. Res. 5
  (2004) 1457--1469.
\newline\urlprefix\url{https://www.jmlr.org/papers/v5/hoyer04a.html}

\bibitem{edelsbrunner2010computational}
H.~Edelsbrunner, J.~Harer, Computational Topology: An Introduction, American
  Mathematical Society, Providence, RI, 2010.
\newblock \href {https://doi.org/10.1090/mbk/069} {\path{doi:10.1090/mbk/069}}.

\bibitem{cohen2007stability}
D.~Cohen-Steiner, H.~Edelsbrunner, J.~Harer, Stability of persistence diagrams,
  Discrete Comput. Geom. 37~(1) (2007) 103--120.
\newblock \href {https://doi.org/10.1007/s00454-006-1276-5}
  {\path{doi:10.1007/s00454-006-1276-5}}.

\bibitem{carriere2021optimizing}
M.~Carriere, F.~Chazal, M.~Glisse, Y.~Ike, H.~Kannan, Y.~Umeda,
  \href{https://proceedings.mlr.press/v139/carriere21a.html}{Optimizing
  persistent homology based functions}, in: M.~Meila, T.~Zhang (Eds.),
  Proceedings of the 38th International Conference on Machine Learning, Vol.
  139 of Proceedings of Machine Learning Research, PMLR, 2021, pp. 1294--1303.
\newline\urlprefix\url{https://proceedings.mlr.press/v139/carriere21a.html}

\bibitem{lee2000algorithms}
D.~D. Lee, H.~S. Seung,
  \href{https://proceedings.neurips.cc/paper_files/paper/2000/file/f9d1152547c0bde01830b7e8bd60024c-Paper.pdf}{Algorithms
  for non-negative matrix factorization}, in: T.~Leen, T.~Dietterich, V.~Tresp
  (Eds.), Advances in Neural Information Processing Systems, Vol.~13, MIT
  Press, 2000.
\newline\urlprefix\url{https://proceedings.neurips.cc/paper_files/paper/2000/file/f9d1152547c0bde01830b7e8bd60024c-Paper.pdf}

\bibitem{hu2019topology}
X.~Hu, F.~Li, D.~Samaras, C.~Chen,
  \href{https://proceedings.neurips.cc/paper/2019/hash/2d95666e2649fcfc6e3af75e09f5adb9-Abstract.html}{Topology-preserving
  deep image segmentation}, in: Advances in Neural Information Processing
  Systems, Vol.~32, Curran Associates, Inc., 2019, pp. 5657--5668.
\newline\urlprefix\url{https://proceedings.neurips.cc/paper/2019/hash/2d95666e2649fcfc6e3af75e09f5adb9-Abstract.html}

\bibitem{clough2022topological}
J.~R. Clough, N.~Byrne, I.~Oksuz, V.~A. Zimmer, J.~A. Schnabel, A.~P. King, A
  topological loss function for deep-learning based image segmentation using
  persistent homology, {IEEE} Trans. Pattern Anal. Mach. Intell. 44~(12) (2022)
  8766--8778.
\newblock \href {https://doi.org/10.1109/TPAMI.2020.3013679}
  {\path{doi:10.1109/TPAMI.2020.3013679}}.

\bibitem{gillis2020nonnegative}
N.~Gillis, Nonnegative Matrix Factorization, Data Science, SIAM, Philadelphia,
  PA, 2020.
\newblock \href {https://doi.org/10.1137/1.9781611976410}
  {\path{doi:10.1137/1.9781611976410}}.

\bibitem{cai2010graph}
D.~Cai, X.~He, J.~Han, T.~S. Huang, Graph regularized nonnegative matrix
  factorization for data representation, IEEE Trans. Pattern Anal. Mach.
  Intell. 33~(8) (2011) 1548--1560.
\newblock \href {https://doi.org/10.1109/TPAMI.2010.231}
  {\path{doi:10.1109/TPAMI.2010.231}}.

\bibitem{li2001learning}
S.~Z. Li, X.~W. Hou, H.~J. Zhang, Q.~S. Cheng, Learning spatially localized,
  parts-based representation, in: Proceedings of the 2001 IEEE Computer Society
  Conference on Computer Vision and Pattern Recognition (CVPR 2001), Vol.~1,
  IEEE, 2001, pp. I--207--I--212.
\newblock \href {https://doi.org/10.1109/CVPR.2001.990477}
  {\path{doi:10.1109/CVPR.2001.990477}}.

\bibitem{yin2010nonnegative}
H.~Yin, H.~Liu, Nonnegative matrix factorization with bounded total variational
  regularization for face recognition, Pattern Recognit. Lett. 31~(16) (2010)
  2468--2473.
\newblock \href {https://doi.org/10.1016/j.patrec.2010.08.001}
  {\path{doi:10.1016/j.patrec.2010.08.001}}.

\bibitem{kim2024supervised}
K.~Kim, H.~Lyu, J.~Kim, J.-H. Jung, Supervised low-rank semi-nonnegative matrix
  factorization with frequency regularization for forecasting spatio-temporal
  data, J. Sci. Comput. 100~(2) (2024) 29.
\newblock \href {https://doi.org/10.1007/s10915-024-02565-7}
  {\path{doi:10.1007/s10915-024-02565-7}}.

\bibitem{bro1998least}
R.~Bro, N.~D. Sidiropoulos, Least squares algorithms under unimodality and
  non-negativity constraints, J. Chemom. 12~(4) (1998) 223--247.
\newblock \href
  {https://doi.org/10.1002/(SICI)1099-128X(199807/08)12:4<223::AID-CEM511>3.0.CO;2-2}
  {\path{doi:10.1002/(SICI)1099-128X(199807/08)12:4<223::AID-CEM511>3.0.CO;2-2}}.

\bibitem{kuang2012symmetric}
D.~Kuang, C.~Ding, H.~Park, Symmetric nonnegative matrix factorization for
  graph clustering, in: Proceedings of the 2012 {SIAM} International Conference
  on Data Mining, SIAM, Philadelphia, PA, 2012, pp. 106--117.
\newblock \href {https://doi.org/10.1137/1.9781611972825.10}
  {\path{doi:10.1137/1.9781611972825.10}}.

\bibitem{obayashi2022persistent}
I.~Obayashi, M.~Kimura, Persistent homology analysis with nonnegative matrix
  factorization for {3D} voxel data of iron ore sinters, JSIAM Lett. 14 (2022)
  151--154.
\newblock \href {https://doi.org/10.14495/jsiaml.14.151}
  {\path{doi:10.14495/jsiaml.14.151}}.

\bibitem{hozumi2024analyzing}
Y.~Hozumi, G.-W. Wei, Analyzing single cell {RNA} sequencing with topological
  nonnegative matrix factorization, J. Comput. Appl. Math. 445 (2024) 115842.
\newblock \href {https://doi.org/10.1016/j.cam.2024.115842}
  {\path{doi:10.1016/j.cam.2024.115842}}.

\bibitem{hensel2021survey}
F.~Hensel, M.~Moor, B.~Rieck, A survey of topological machine learning methods,
  Front. Artif. Intell. 4 (2021) 681108.
\newblock \href {https://doi.org/10.3389/frai.2021.681108}
  {\path{doi:10.3389/frai.2021.681108}}.

\bibitem{gameiro2016continuation}
M.~Gameiro, Y.~Hiraoka, I.~Obayashi, Continuation of point clouds via
  persistence diagrams, Physica D: Nonlinear Phenom. 334 (2016) 118--132.
\newblock \href {https://doi.org/10.1016/j.physd.2015.11.011}
  {\path{doi:10.1016/j.physd.2015.11.011}}.

\bibitem{bruel2020topology}
R.~Br{\"u}el-Gabrielsson, V.~Ganapathi-Subramanian, P.~Skraba, L.~J. Guibas,
  Topology-aware surface reconstruction for point clouds, Comput. Graph. Forum
  39~(5) (2020) 197--207.
\newblock \href {https://doi.org/10.1111/cgf.14079}
  {\path{doi:10.1111/cgf.14079}}.

\bibitem{chen2019topological}
C.~Chen, X.~Ni, Q.~Bai, Y.~Wang,
  \href{https://proceedings.mlr.press/v89/chen19g.html}{A topological
  regularizer for classifiers via persistent homology}, in: Proceedings of the
  22nd International Conference on Artificial Intelligence and Statistics,
  Vol.~89 of Proceedings of Machine Learning Research, PMLR, 2019, pp.
  2573--2582.
\newline\urlprefix\url{https://proceedings.mlr.press/v89/chen19g.html}

\bibitem{moor2020topological}
M.~Moor, M.~Horn, B.~Rieck, K.~Borgwardt,
  \href{https://proceedings.mlr.press/v119/moor20a.html}{Topological
  autoencoders}, in: Proceedings of the 37th International Conference on
  Machine Learning, Vol. 119 of Proceedings of Machine Learning Research, PMLR,
  2020, pp. 7045--7054.
\newline\urlprefix\url{https://proceedings.mlr.press/v119/moor20a.html}

\bibitem{gabrielsson2020topology}
R.~B. Gabrielsson, B.~J. Nelson, A.~Dwaraknath, P.~Skraba,
  \href{https://proceedings.mlr.press/v108/gabrielsson20a.html}{A topology
  layer for machine learning}, in: Proceedings of the 23rd International
  Conference on Artificial Intelligence and Statistics, Vol. 108 of Proceedings
  of Machine Learning Research, PMLR, 2020, pp. 1553--1563.
\newline\urlprefix\url{https://proceedings.mlr.press/v108/gabrielsson20a.html}

\bibitem{hofer2020graph}
C.~Hofer, F.~Graf, B.~Rieck, M.~Niethammer, R.~Kwitt,
  \href{https://proceedings.mlr.press/v119/hofer20b.html}{Graph filtration
  learning}, in: H.~Daum{\'e}, III, A.~Singh (Eds.), Proceedings of the 37th
  International Conference on Machine Learning, Vol. 119 of Proceedings of
  Machine Learning Research, PMLR, 2020, pp. 4314--4323.
\newline\urlprefix\url{https://proceedings.mlr.press/v119/hofer20b.html}

\bibitem{leygonie2022framework}
J.~Leygonie, S.~Oudot, U.~Tillmann, A framework for differential calculus on
  persistence barcodes, Found. Comput. Math. 22~(4) (2022) 1069--1131.
\newblock \href {https://doi.org/10.1007/s10208-021-09522-y}
  {\path{doi:10.1007/s10208-021-09522-y}}.

\bibitem{dey2022computational}
T.~K. Dey, Y.~Wang, Computational Topology for Data Analysis, Cambridge
  University Press, Cambridge, 2022.
\newblock \href {https://doi.org/10.1017/9781009099950}
  {\path{doi:10.1017/9781009099950}}.

\bibitem{korte2018combinatorial}
B.~Korte, J.~Vygen, Combinatorial Optimization: Theory and Algorithms, 6th
  Edition, Vol.~21 of Algorithms and Combinatorics, Springer, Berlin, 2018.
\newblock \href {https://doi.org/10.1007/978-3-662-56039-6}
  {\path{doi:10.1007/978-3-662-56039-6}}.

\bibitem{bauer2021ripser}
U.~Bauer, {Ripser}: efficient computation of {V}ietoris--{R}ips persistence
  barcodes, J. Appl. Comput. Topol. 5~(3) (2021) 391--423.
\newblock \href {https://doi.org/10.1007/s41468-021-00071-5}
  {\path{doi:10.1007/s41468-021-00071-5}}.

\bibitem{shamir2004cluster}
R.~Shamir, R.~Sharan, D.~Tsur, Cluster graph modification problems, Discrete
  Appl. Math. 144~(1--2) (2004) 173--182.
\newblock \href {https://doi.org/10.1016/j.dam.2004.01.007}
  {\path{doi:10.1016/j.dam.2004.01.007}}.

\bibitem{perea2015sliding}
J.~A. Perea, J.~Harer, Sliding windows and persistence: An application of
  topological methods to signal analysis, Found. Comput. Math. 15~(3) (2015)
  799--838.
\newblock \href {https://doi.org/10.1007/s10208-014-9206-z}
  {\path{doi:10.1007/s10208-014-9206-z}}.

\bibitem{adamaszek2017vietoris}
M.~Adamaszek, H.~Adams, The {V}ietoris--{R}ips complexes of a circle, Pac. J.
  Math. 290~(1) (2017) 1--40.
\newblock \href {https://doi.org/10.2140/pjm.2017.290.1}
  {\path{doi:10.2140/pjm.2017.290.1}}.

\bibitem{clarke1983optimization}
F.~H. Clarke, Optimization and Nonsmooth Analysis, Wiley, New York, 1983.

\bibitem{davis2020stochastic}
D.~Davis, D.~Drusvyatskiy, S.~Kakade, J.~D. Lee, Stochastic subgradient method
  converges on tame functions, Found. Comput. Math. 20~(1) (2020) 119--154.
\newblock \href {https://doi.org/10.1007/s10208-018-09409-5}
  {\path{doi:10.1007/s10208-018-09409-5}}.

\bibitem{kaji2020cubical}
S.~Kaji, T.~Sudo, K.~Ahara, {Cubical Ripser}: Software for computing persistent
  homology of image and volume data, preprint (2020).
\newblock \href {http://arxiv.org/abs/2005.12692} {\path{arXiv:2005.12692}},
  \href {https://doi.org/10.48550/arXiv.2005.12692}
  {\path{doi:10.48550/arXiv.2005.12692}}.

\bibitem{boutsidis2008nndsvd}
C.~Boutsidis, E.~Gallopoulos, {SVD} based initialization: A head start for
  nonnegative matrix factorization, Pattern Recognit. 41~(4) (2008) 1350--1362.
\newblock \href {https://doi.org/10.1016/j.patcog.2007.09.010}
  {\path{doi:10.1016/j.patcog.2007.09.010}}.

\bibitem{gudhi:urm}
{The GUDHI Project}, \href{https://gudhi.inria.fr/doc/3.13.0/}{{GUDHI} User and
  Reference Manual}, {3.13.0} Edition, GUDHI Editorial Board, 2026.
\newline\urlprefix\url{https://gudhi.inria.fr/doc/3.13.0/}

\bibitem{scikit-learn}
F.~Pedregosa, G.~Varoquaux, A.~Gramfort, V.~Michel, B.~Thirion, O.~Grisel,
  M.~Blondel, P.~Prettenhofer, R.~Weiss, V.~Dubourg, J.~Vander{P}las,
  A.~Passos, D.~Cournapeau, M.~Brucher, M.~Perrot, E.~Duchesnay,
  \href{https://jmlr.org/papers/v12/pedregosa11a.html}{Scikit-learn: Machine
  learning in {P}ython}, J. Mach. Learn. Res. 12 (2011) 2825--2830.
\newline\urlprefix\url{https://jmlr.org/papers/v12/pedregosa11a.html}

\bibitem{fournet2014contact}
J.~Fournet, A.~Barrat, Contact patterns among high school students, PLoS ONE
  9~(9) (2014) e107878.
\newblock \href {https://doi.org/10.1371/journal.pone.0107878}
  {\path{doi:10.1371/journal.pone.0107878}}.

\bibitem{openmeteo}
{Open-Meteo}, Open-meteo historical weather {API},
  \url{https://open-meteo.com/en/docs/historical-weather-api}, online
  documentation; accessed 2026-09-08 (2026).

\end{thebibliography}

\ifdefined\arxivversion
\else
\submissionbiographies
\fi

\end{document}